\label{key}

\documentclass{ws-aa}
\usepackage{amsmath}
\usepackage{amsfonts}
\usepackage{graphicx}
\usepackage[center]{subfigure}
\usepackage{multirow}

\usepackage{algorithm}
\usepackage{algorithmic}
\begin{document}

\markboth{Jia Cai et al.}{SStaGCN: Simplified stacking based graph convolutional networks.}

%
\catchline{}{}{}{}{}
%

\title{SStaGCN: Simplified stacking based graph convolutional networks}

\author{Jia Cai}

\address{School of Digital Economy,  Guangdong University  of Finance $\&$ Economics\\
	Guangzhou, 510320,
	P. R. China.\\
	jiacai1999@gdufe.edu.cn}

\author{Zhilong Xiong}
\address{School of Statistics and Mathematics,  Guangdong University  of Finance $\&$ Economics\\
	Guangzhou, 510320,
	P. R. China.\\
	zhilongxiong98@outlook.com}

\author{Shaogao Lv\footnote{Corresponding author.}}
\address{Department of Statistics and Data Science, Nanjing Audit University\\
	 Nanjing,  211815, 
	 P. R. China.\\
	 lvsg716@nau.edu.cn}
	
\maketitle

\begin{history}
\received{(Day Month Year)}
\revised{(Day Month Year)}
\end{history}

\begin{abstract}
Graph Convolutional Networks (GCNs) are powerful models extensively studied for various graph-structured data learning tasks. However, designing GCNs that effectively mitigate the over-smoothing phenomenon remains a crucial challenge. In this paper, we propose a novel Simplified Stacking-based GCN (SStaGCN) framework, leveraging stacking and aggregation techniques to address different types of graph-structured data. Specifically, we first employ stacking base models to extract node features from the graph. Next, we apply aggregation methods—such as mean, attention, and voting techniques—to further enhance feature extraction capabilities. These refined node features are then fed into a vanilla GCN model. Additionally, we provide a theoretical analysis of the generalization bound for the proposed model. Extensive experiments on three public citation networks and three heterogeneous tabular datasets demonstrate the effectiveness and efficiency of the SStaGCN approach over several state-of-the-art GCNs. Notably, SStaGCN effectively mitigates the over-smoothing issue common in GCNs.
\end{abstract}

\keywords{Graph convolutional network; stacking; aggregation.}

\ccode{Mathematics Subject Classification 2010: 68T05, 68W40}

\section{Introduction}\label{sec:intr}
Recent research on learning from graph-structured data has gained considerable attention across diverse fields within artificial intelligence. Graph Neural Networks (GNNs), particularly Graph Convolutional Networks (GCNs) \cite{Bruna2014SpectralNA, Gori2005ANM, Hamilton2017InductiveRL}, have achieved remarkable success in modeling graph-structured data and have been widely applied to recommendation systems \cite{Sun2020AFF}, computer vision \cite{Casas2020SpAGNNSG}, molecular design \cite{2020stokes}, natural language processing \cite{Yao2019GraphCN}, node classification \cite{Kipf2017SemiSupervisedCW}, and clustering tasks \cite{Zhu2012MaxMarginNL}.

Despite these successes, real-world applications present diverse types of graph-structured data, raising the question: can we design a general framework to handle distinct types of graph-structured data in a simpler, more effective manner? Additionally, as discussed in \cite{li2018Deeper}, the graph convolution operation in GCNs is a specific form of Laplacian smoothing, which mixes node features across different clusters and can lead to over-smoothing \cite{li2018Deeper}.

Sun et al. \cite{Sun2019AdaGCNAG} addressed this issue by developing an RNN-like GCN with AdaBoost, enabling the extraction of knowledge from high-order neighbors of current nodes. However, the relation $A^\ell = B^\ell$ does not necessarily imply $A = B$. A simple example illustrating this is the case where $\ell = 2$ (two layers), with 
$$A=
\left(\begin{matrix}
0& 1  \\
1 &0
\end{matrix}
\right),
\qquad
B=
\left(\begin{matrix}
1& 0  \\
0 &1
\end{matrix}
\right).
$$
As a result, this RNN-like GCN inevitably loses crucial information about the original graph structure, leaving the challenge of designing a simple GCN model to tackle the over-smoothing issue unresolved.

The stacking method \cite{Deroski2004IsCC}, also known as stacked generalization, is an ensemble approach that combines multiple classifiers using base models and a meta-model. Stacking has shown success in feature extraction tasks across various data types, including disordered or irregular data, due to its advantageous properties:

By combining distinct learners in the base models, stacking effectively captures discernible features. Base models are trained on the entire training dataset, allowing for robust performance evaluation on test data. A meta-model then aggregates these outputs to make final predictions on the test data. Combining stacking with GCN models can yield significant benefits. In this paper, we propose a novel simplified stacking-based architecture, Simplified Stacking-based GCN (SStaGCN), to address the over-smoothing issue in GCNs. SStaGCN integrates the feature extraction strengths of stacking with GCNs' graph learning capabilities. This synergy enables SStaGCN to harness the strengths of both stacking and GCNs.  

The main contributions of this paper are as follows:

\begin{itemlist}

\item  We propose a novel, versatile architecture that combines simplified stacking and GCNs, designed to adapt flexibly to diverse types of graph-structured data, including heterogeneous tabular data\footnote{Heterogeneous tabular data includes a mix of discrete and continuous variables.}.

\item We present a generalization bound analysis to elucidate the contributions of stacking and aggregation from a learning theory perspective.

\item We conduct extensive evaluations of our approach against strong baselines on node prediction tasks. The experimental results demonstrate significant performance improvements in both homogeneous and heterogeneous node classification tasks across various real-world graph-structured datasets, effectively alleviating the over-smoothing phenomenon.

\end{itemlist}
The remainder of the paper is organized as follows. In Section \ref{relaw}, we provide a brief review of related work. Section \ref{main} presents the theoretical analysis and proposed algorithm for GCNs. In Section \ref{experi}, we detail experiments conducted on three public citation networks and three heterogeneous tabular datasets. Section \ref{concl} offers discussions and concluding remarks. The proof of the main result is included in the Appendix.

\section{Related Work}\label{relaw}

{\bf Graph Convolutional Networks}

GCNs are commonly understood as extensions of traditional convolutional neural networks applied to graph domains. Generally, there are two types of GCNs \cite{Michael2017Geometric}: spatial GCNs and spectral GCNs. Spatial GCNs construct new feature vectors for each node by leveraging neighborhood information, where convolution acts as a "patch operator." Spectral GCNs, on the other hand, define convolution by decomposing a graph signal in the spectral domain and applying a spectral filter (such as Fourier or wavelet filters, \cite{Bruna2014SpectralNA, Xu2019GraphWN, lim2020Fast}) to the spectral components \cite{2013Discrete, 2013Theshuman}. However, spectral GCNs require computation of the Laplacian eigenvectors, which becomes impractical for large-scale graphs due to its high computational cost.

To address this, Hammond et al. \cite{2011Wavelets} used Chebyshev polynomials up to the $K$-th order to approximate the spectral filter. Defferrard and Vandergheynst \cite{Defferrard2016ConvolutionalNN} proposed the $K$-localized ChebyNet, and Kipf and Welling \cite{Kipf2017SemiSupervisedCW} introduced a simpler model by setting $K=1$, which proved effective for semi-supervised classification tasks. Wu et al. \cite{2019wuSimplifying} further simplified GCNs by removing nonlinearities and collapsing the weight matrices between layers. In contrast, other works like \cite{li2019DeepGCNs, li2018Deeper} explored the development of deeper GCNs, while multi-scale deep GCNs were examined in \cite{luan2019Break}. Despite these advancements, over-smoothing remains a significant challenge for GCNs.

Various strategies have been proposed to address over-smoothing. Klicpera et al. \cite{klicpera_predict_2019} and Chien et al. \cite{chien2021adaptive} used PageRank and generalized PageRank, respectively, to update graph information. DropEdge \cite{rong2020dropedge} randomly removes edges in the graph to mitigate over-smoothing. Similarly, GRAND \cite{feng2020graph} introduced a random propagation strategy to augment graph data and applied consistency regularization. Chen et al. \cite{chen2020simple} developed GCNII, a GCN variant that uses initial residuals and identity mapping to address over-smoothing, while DGMLP \cite{zhang2021evaluating} incorporated adaptive modes and residual connections.

This paper also adopts a decoupled approach, performing feature extraction followed by message propagation to classify nodes. However, our approach achieves superior performance and accuracy while providing a more flexible and general framework.

{\bf Ensemble learning based graph neural networks}

Sun et al. \cite{Sun2019AdaGCNAG} designed an RNN-like graph structure to extract information from high-order neighbors of each node, while Ivanov and Prokhorenkova \cite{Ivanov2021BoostTC} integrated gradient-boosted decision trees (GBDT) into GNNs, developing BGNN to handle heterogeneous tabular data. This raises a natural question: can we design a general GCN architecture that not only accommodates diverse graph data but also addresses the over-smoothing issue? This paper seeks to investigate this challenge and provides a solution to the above question.
\section{The proposed approach: SStaGCN}\label{main}

\subsection{Graph convolutional networks}
Given an undirected graph $\mathcal{G} = (\mathcal{V}, \mathcal{E})$ with nodes $v_i \in \mathcal{V}$ and edges $(v_i, v_j) \in \mathcal{E}$, let $A \in \mathbb{R}^{N \times N}$ denote the adjacency matrix, where $N$ is the total number of nodes. The corresponding degree matrix is denoted as $\mathbf{D}$, with $\mathbf{D}_{ii} = \sum_{j} A_{ij}$. For an undirected graph, it is evident that $A_{ij} = A_{ji}$.

In conventional GCN models for semi-supervised node classification, the node embeddings with two convolutional layers are computed as follows:
\begin{equation}\label{GCNexpre}
Z=\hat{A} {\rm ReLU}(\hat{A}XW^{(0)} )W^{(1)}.
\end{equation}
Here, $Z \in \mathbb{R}^{N \times K}$ represents the final embedding matrix (output logits) of the nodes prior to the softmax operation, where $K$ is the number of classes. The feature matrix $X \in \mathbb{R}^{N \times d}$ contains node features, with $d$ as the input dimension. We define $\hat{A} = \tilde{D}^{-\frac{1}{2}} \tilde{A} \tilde{D}^{-\frac{1}{2}}$, where $\tilde{A} = A + I$ (with $I$ as the identity matrix) and $\tilde{D}$ is the degree matrix of $\tilde{A}$.

Furthermore, $W^{(0)} \in \mathbb{R}^{d \times H}$ is the weight matrix for the input-to-hidden layer with $H$ hidden features, and $W^{(1)} \in \mathbb{R}^{H \times K}$ is the weight matrix for the hidden-to-output layer. Notably, the standard GCN model in Eq. (\ref{GCNexpre}) applies $\hat{A}$ to $X$ repeatedly, which leads to all node features becoming indistinguishable due to excessive smoothing.

\subsection{Stacking}
Stacking is a well-known and widely used ensemble machine learning algorithm that integrates models in a hierarchical framework \cite{Wolpert1992StackedG}. It employs a meta-learning algorithm to optimally combine predictions from two or more base machine learning models. A traditional stacking model consists of multiple base models and a meta-model that integrates the base models' predictions. Each base model is trained on the dataset and produces individual predictions. The meta-model then learns to best combine these predictions, typically using a straightforward approach to provide a cohesive interpretation of the base models' outputs. As a result, linear models, such as linear regression for regression tasks and logistic regression for classification tasks, are often chosen as meta-models.

\subsection{The proposed approach}
As noted above, GCNs may blend node features from different clusters, which can lead to suboptimal predictions. Therefore, it is essential to aggregate node information effectively to improve predictive accuracy. Inspired by the traditional stacking approach and the work in \cite{Ivanov2021BoostTC, Sun2019AdaGCNAG}, we aim to reduce computational cost by using only the base models from the stacking method and then aggregating their outputs to derive the nodes' attributes. Specifically, our proposed method operates as follows: in the first layer, we obtain  $X'$ by passing  $X\in \mathbb{R}^{N\times d}$  (the feature matrix) through 
$k$ base classifiers.
\begin{equation}
X'_i=h_i (X) , \qquad  i =1,\cdots,k,
\end{equation}
Next, we obtain the preliminary classification results $X'_i$ by passing $X$ through each base classifier  $h_i(X)$ for $i=1,\cdots,k$. Then, we apply an aggregation method to produce the final output results, i.e.,
\begin{equation}
{\tilde X} = g(X'_1,\cdots,X'_k),
\end{equation}
where $g(\cdot)$ denotes the aggregation method, which is designed to combine attribute values into a single representative value. Common aggregation methods include mean, attention, or voting approaches.
\begin{figure*}[!htb]
	\centering
	\includegraphics[scale = 0.38]{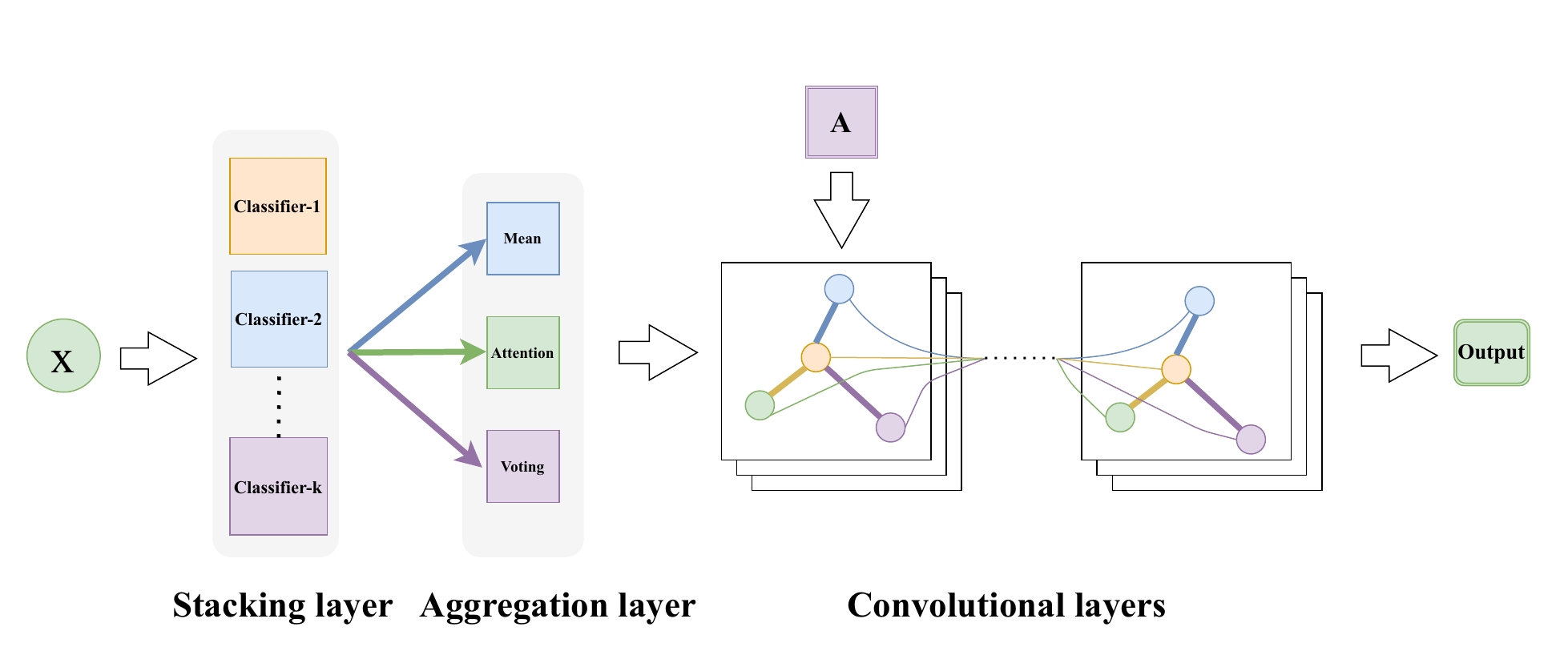}
	\caption{Workflow of the  SStaGCN model.}
	\label{pipSTGCN}
\end{figure*}

\begin{algorithm}[!h]
	\renewcommand{\algorithmicrequire}{\textbf{Input:}}
	\renewcommand\algorithmicensure {\textbf{Output:} }
	\caption{SStaGCN.}
	\label{SStaGCN}
	\begin{algorithmic}[1]
		\REQUIRE ~~\\
		Feature matrix $X$, normalized adjacency matrix $\hat{A}$, graph $\mathcal{G}$, base classifier $h_i(X)$, $i=1,\cdots,k$, aggregation method $g(\cdot)$;\\
		\ENSURE ~~\\
		Final predictor $f(X)$;\\
		\STATE Attain $X'_i$ ($i=1,\cdots,k$) via $k$  base classifiers;\\
		$X'_i \gets h_i(X)$, $i=1,\cdots,k$;
		\STATE Aggregate $X'_1,\cdots,X'_k$;\\
		${\tilde X} \gets g(X'_1,\cdots,X'_k)$;
		\STATE Feed ${\tilde X}$ into vanilla GCN \cite{Kipf2017SemiSupervisedCW};\\
		$f(X)\gets GCN({\tilde X} )$;
		\RETURN $f(X)$;
	\end{algorithmic}
\end{algorithm}

\textbf{Mean}:  mean operator takes the element-wise mean of the components $X'_1,\cdots, X'_k$.

\textbf{Attention} \cite{Vaswani2017AttentionIA}: The attention mechanism has been widely applied in various fields of deep learning, including image processing \cite{Mnih2014RecurrentMO}, speech recognition \cite{Bahdanau2015NeuralMT}, and natural language processing \cite{Yin2016ABCNNAC}. This concept is inspired by the way humans focus attention selectively. Let the query vector $V$ represent the output of the base classifiers, and let the query vector $Y$ denote the data label. We then compute the attention coefficients between $Y$ and $V$ as follows:
\begin{equation}
a_{i} ={{\rm softmax}}({\rm cos}(Y,V_{i} )),\qquad i=1,\cdots,k,
\end{equation}
where $\cos$ denotes cosine similarity. Finally, the input ${\tilde X}$ of the graph convolutional layer is aggregated by considering the following sum with attention scores $a_{i}$, $i=1,\cdots, k$.
\begin{equation}
{\tilde X} =\sum^k_{i=1} a_{i}V_{i}.
\end{equation}

\textbf{Voting} \cite{Schapire1997BoostingTM}: The voting method is the most straightforward approach in ensemble learning, aiming to select one or more ``winning" predictions. In this work, our goal is to select the most common prediction among the outputs of the base classifiers, using a majority voting approach. In cases where categories appear with equal frequency, a category will be chosen at random.

This brings us to a novel GCN model, SStaGCN, specifically designed to handle diverse types of graph-structured data by effectively integrating stacking, aggregation, and the vanilla GCN model \cite{Kipf2017SemiSupervisedCW}. In SStaGCN, the first layer leverages base models from the stacking approach, while the second layer applies an aggregation method—such as mean, attention, or voting—to enhance the feature extraction capabilities of standard GCN models. The aggregated data is then used as input to the convolutional layer, which produces the final predictions. The workflow of the proposed model is presented in Algorithm \ref{SStaGCN} and Fig. \ref{pipSTGCN}.  This prompts the question: is there a theoretical guarantee for the proposed approach?

\subsection{Generalization bound}
In this section, we provide a theoretical generalization analysis of the proposed approach. In the following analysis, we assume that both the adjacency matrix 
$A$ and the feature matrix $X$ are fixed.

In learning theory, the risk of a function $f$ over the unknown population distribution $\mathbb{P}$ is measured by
$$
\mathcal{E}(f):=\mathbb{E}_{\mathcal{X}\times \mathcal{Y}}[L(y,f(\tilde{\bf x}))],
$$
where $L$ is the loss function defined as a map: $L: \mathcal{X}\times \mathcal{Y}\rightarrow \mathbb{R}^+$,  Given a  training data and  adjacency matrix $A$, the objective is to estimate parameters $(W^{(0)},W^{(1)})$ from model \eqref{GCNexpre} based on empirical data. Concretely, we attempt to  minimize the empirical risk functional over some function class  $\mathcal{F}$,
$$
\mathcal{E}_m(f):=\frac{1}{m}\sum_{j=1}^mL(y_j,f(\tilde{\bf x}_j)),
$$
where  $\{(\tilde{{\bf x}}_j,y_j)^m_{j=1}\}$ is  the labeled sample achieved from the original training data  $\{({\bf x}_j,y_j)^m_{j=1}\}$ via stacking and aggregation.  Typically, the clustering algorithms will only produce discernible node features. Hence,  if $\|{\bf x}_i\|\leq R, i=1,\cdots, N$, the stacking and aggregation mechanisms will not make $\tilde {\bf x}_i$, $i=1,\cdots, N$  violate the constraints, i.e., $\|\tilde {\bf x}_i\|\leq R, i=1,\cdots, N$.
Now we are in position to present the theoretical generalization bound.
\begin{theorem}\label{mainthm}
	Suppose $\|{\bf x}_i\|_2\leq R, i=1,\cdots, N$,  $\|W^{(0)}\|_F\leq B_1$, $\|W^{(1)}\|_F\leq B_2$. Denote $N(v)$ the number of neighbors of node $v\in \Omega$ (the set of node indices with observed labels), let $q=\max\{N(v)\}$, $f: \mathcal{X}\rightarrow \mathbb{R}$ be any given predictor of a class for GCNs with one-hidden layer. Assume that the  loss function $L	 (y,\cdot)$ is Lipschitz continuous with Lipschitz constant $\alpha_L$. Then, for any $\delta>0$, with probability at least $1-\delta$,  we have
	$$
	\mathcal{E}(f)\leq \mathcal{E}_N(f)+\frac{2\alpha_L\sqrt{2q}KB_1B_2R\sum^q_{s=1}M_s}{\sqrt{m}} +\sqrt{\frac{2\log(2/\delta)}{N}},
	$$
\end{theorem}
where $\|\cdot\|_F$ is the Frobenius norm, $M_s =\max_{i\in [N]} |\hat A_{iv}|$ with $v\in N(i)$.
\begin{remark}
Theorem \ref{mainthm} indicates that the dominant upper bound depends linearly on the maximum number of neighbors 	$q$, as well as on the weight bounds $B_1$ and $B_2$, which are strongly influenced by the dimension $d_1$. Clearly, setting $d_1\ll d$ results in a tighter generalization bound. In the binary classification case ($K=1$), the result presented here is similar to that in \cite{shaogao2021}. Overall, the conditions stated in Theorem \ref{mainthm} are mild and reasonable.
\end{remark}
\section{Experiments}\label{experi}
\subsection{Datasets}
To evaluate the performance of the proposed SStaGCN model across various types of graph-structured data, we utilize six real-world datasets for a semi-supervised node classification task. This includes three commonly used citation networks—Cora, CiteSeer, and Pubmed \cite{Sen2008CollectiveCI}—and three heterogeneous tabular datasets: House$\_$class, VK$\_$class, and DBLP \cite{Ivanov2021BoostTC}. These heterogeneous tabular datasets differ from those in \cite{chuanshi}, containing graph data with diverse edge and node types.

In the citation networks, nodes represent documents, and undirected edges denote citation relationships between documents. Node features correspond to representative words in the documents, while the label rate indicates the percentage of node labels used for training. The Cora dataset comprises $2708$ nodes, $5429$ edges, $7$ classes, and $1433$ node features; CiteSeer has $3327$ nodes, $4732$ edges, $6$ classes, and $3703$ node features; and Pubmed includes $19717$ nodes, $4438$ edges, and $3$ classes. We use $140$, $120$, and $60$ nodes for training in Cora, CiteSeer, and Pubmed, respectively, and allocate $1000$ nodes for testing and $500$ nodes for cross-validation. This data split matches that used in GCN, Graph Attention Network (GAT, \cite{Velickovic2018GraphAN}), and GWNN \cite{Xu2019GraphWN}.

According to \cite{Ivanov2021BoostTC},  House$\_$class and VK$\_$class datasets are derived from House and VK datasets, where target labels are converted into discrete classes due to limited availability of publicly accessible heterogeneous graph-structured data. In these heterogeneous tabular datasets, features are independently defined and vary in type, scale, and meaning. The VK dataset represents a popular social network, with node features that are both numerical (e.g., last active time) and categorical (e.g., country and university affiliation). Similar to \cite{Ivanov2021BoostTC}, we categorize age into bins: $<20$, $20–25$, $25–30$, $\cdots$, $45–50$, $>50$. For the House dataset, we categorize target values within the range $[1.0, 1.5, 2.0, 2.5]$, resulting in five classes for House$\_$class and seven for VK$\_$class. In the DBLP dataset, we construct a single graph by focusing on the APA (author-paper-author) meta-path. Each dataset is divided into training, validation, and testing splits in a $0.6/0.2/0.2$ ratio across five random seeds.

Details about the citation networks and heterogeneous tabular datasets are presented in Table 1 and Table 2, respectively.

\begin{table}[th]

\tbl{Summary of the citation networks.}
{\begin{tabular}{@{}lccc@{}} \toprule
		Dataset   &Cora &CiteSeer & Pubmed \\
\colrule
		Nodes   & 2708 &3327 &19717\\
		Edges   &5429 &4732 &44338\\
		Features   &1433 &3703 &500\\
		Classes   & 7 & 6 &3\\
		Label Rate   &5.2\% &3.6\% &0.3\% \\
\botrule
\end{tabular}}
\label{citdataset}
\end{table}

\begin{table}[th]
\tbl {Summary of the heterogeneous tabular datasets.}
{\begin{tabular}{@{}lccc@{}} 
		\toprule
		Dataset   &House$\_$class&VK$\_$class & DBLP\\
\colrule
		Nodes   &20640    &54028     &14475   \\
		Edges   &182146    &213644     &40269   \\
		Features   &6    &14     &5002   \\
		Classes   &5    &7     &4   \\
		Min Target Nodes  &0.14    &13.48     &745   \\
		Max Target Nodes  &5.00    &118.39     &1197   \\
\botrule
	\end{tabular}}
\label{hetedataset}
\end{table}

\subsection{Baselines}
We compare SStaGCN with four classical graph convolutional networks—ChebyNet, GCN \cite{Kipf2017SemiSupervisedCW}, GAT, and APPNP \cite{Klicpera2019PredictTP}—as well as two ensemble-based GCN models: AdaGCN and BGNN.
\subsection{Setting}
As shown in Fig. \ref{pipSTGCN}, the proposed SStaGCN model consists of four layers, with the first and second layers referred to as the stacking and aggregation layers, respectively. The stacking layer utilizes base models from the stacking method, incorporating a combination of seven classical classifiers: KNN, Random Forest, Naive Bayes, Decision Tree, SVC \cite{Pal2020DataCW}, GBDT \cite{Friedman2001GreedyFA}, and Adaboost \cite{Freund1999ASI}. These classifiers are widely used in classical machine learning for their strengths in handling different task types.

In the aggregation layer, we employ three aggregation methods: mean, attention, and voting. For the mean approach, we calculate the average output from the stacking layer, rounding it as needed. In the attention mechanism, we treat the label data as vector $Y$, use the stacking layer's predicted output as the query vector 
$V$, and compute attention coefficients accordingly. For the voting method, we apply hard voting from ensemble learning \cite{Friedhelm2013Ensemble}.

The output of the aggregation layer serves as the input to the first layer of a two-layer GCN. In our configuration, the GCN has 16 hidden units and uses the Adam optimizer \cite{Kingma2015AdamAM} by default, with cross-entropy as the loss function. The learning rate is set to $\gamma=0.01$, the number of iterations  ${\rm itr} =500$,  weight decay at  $5e-4$, and a dropout rate of $0$.

\subsection{Results}
\begin{table}[th]
	\tbl {Average accuracy on $3$ citation networks under $30$ runs by computing the $95\%$ confidence interval via bootstrap.}
	{\begin{tabular}{@{}lccc@{}} \toprule
		Method   &Cora&CiteSeer&Pubmed\\
\colrule
		ChebyNet	&81.20$\pm0.00$   &69.80$\pm$0.00   &74.40$\pm$0.00\\
		GCN   &81.50$\pm$  0.00    &70.30$\pm$0.00   &79.00$\pm$0.00\\
		GAT   &83.00$\pm$0.70   &72.50$\pm$0.70   &79.00$\pm$0.30\\
		APPNP   &85.09$\pm$0.25   &75.73$\pm$0.30    &79.73$\pm$0.31\\
		\hline
		AdaGCN   &85.97$\pm$0.20     &76.68$\pm$0.20    &79.95$\pm$0.21           \\
		BGNN     & 41.97$\pm$0.19   &30.74$\pm$0.10       &10.32$\pm$0.10         \\
\colrule
		Sim$\_$Stacking     &43.19$\pm$2.05&62.70$\pm$0.53&87.70$\pm$0.23\\
		SStaGCN (Mean)   &90.35$\pm$0.20   &86.40$\pm$0.12      &82.30$\pm$0.19\\
		SStaGCN (Attention)   &91.60$\pm$0.18   &87.20$\pm$0.12   &82.40$\pm$0.23\\
		SStaGCN (Voting)  &\textbf{93.10$\pm$0.16}&\textbf{88.70$\pm$0.14}    &\textbf{92.07$\pm$0.20}\\
\botrule
	\end{tabular}}
	\label{accuracy_cit}
\end{table}

\begin{table}[th]
	\tbl {Average accuracy on heterogeneous tabular datasets under $30$ runs by calculating the $95\%$ confidence interval via bootstrap.}
	{\begin{tabular}{@{}lccc@{}} \toprule
		Method   &House$\_$class  &   VK$\_$class &  DBLP \\
\colrule
		ChebyNet	&54.74$\pm$0.10       &57.19$\pm$0.36              &32.14$\pm$0.00        \\
		GCN   &55.07$\pm$0.13    &56.40$\pm$0.09         &39.49$\pm$1.37         \\
		GAT   &56.50$\pm$0.22  &56.42$\pm$0.19        &76.83$\pm$0.78         \\
		APPNP   &57.03$\pm$0.27    &56.72$\pm$0.11              &79.47$\pm$1.46         \\
		
\colrule
		AdaGCN   &26.20$\pm$0.00    &46.00$\pm$0.00                   &10.06$\pm$0.00              \\
		BGNN     &66.70$\pm$0.27           &66.32$\pm$0.20      &86.94$\pm$0.74         \\
\colrule
		Sim$\_$Stacking   &53.89$\pm$0.29  &56.64$\pm$0.10     &71.58$\pm$0.64         \\
		SStaGCN (Mean)   &72.35$\pm$0.05       &66.62$\pm$0.17      &82.31$\pm$0.20         \\
		SStaGCN (Attention)  &72.40$\pm$0.12             &77.64$\pm$0.08 &82.51$\pm$0.22         \\
		SStaGCN (Voting)  &\textbf{76.13$\pm$0.12}   &\textbf{87.92$\pm$0.07}    &\textbf{92.60$\pm$0.10}  \\
\botrule
	\end{tabular}}
	\label{accuracy_het}
\end{table}

\begin{table}[th]
	\tbl {Average F1-score (macro) on $3$ citation networks under $30$  runs by computing the $95\%$ confidence interval via bootstrap.}
{\begin{tabular}{@{}lccc@{}} \toprule

		Method   &Cora   &CiteSeer    &Pubmed\\
\colrule
		ChebyNet	&77.99$\pm$0.54   &63.76$\pm$0.34        &77.74$\pm$0.42         \\
		GCN   &82.89$\pm$0.30   &70.65$\pm$0.37         &78.83$\pm$0.32         \\
		GAT   &83.59$\pm$0.25   &70.62$\pm$0.29         &77.77$\pm$0.40         \\
		APPNP   &84.29$\pm$0.22    &71.05$\pm$0.38  &79.66$\pm$0.31         \\
\colrule
		AdaGCN   &79.55$\pm$0.19      &63.62$\pm$0.19       &78.55$\pm$0.21         \\
		BGNN     &40.81$\pm$0.25       &32.73$\pm$0.13      &8.46$\pm$0.08         \\
\colrule
		Sim$\_$Stacking   &44.02$\pm$1.61          &60.86$\pm$0.56              &87.31$\pm$0.13\\
		SStaGCN (Mean)     &90.66$\pm$0.18          &86.42$\pm$0.12             &82.30$\pm$0.19\\
		SStaGCN (Attention)  &91.69$\pm$0.14        &87.24$\pm$0.14             &82.45$\pm$0.23\\
		SStaGCN (Voting)     &\textbf{92.76$\pm$0.16}   &\textbf{88.73$\pm$0.14} &\textbf{92.07$\pm$0.20}\\
\botrule
	\end{tabular}}
	\label{f1 score_cit}
\end{table}

\begin{table}[th]
	\tbl {Average F1-score(macro) on heterogeneous tabular datasets  under $30$ runs by calculating the $95\%$ confidence interval via bootstrap.}
{\begin{tabular}{@{}lccc@{}} \toprule

		Method   &House$\_$class  &   VK$\_$class &  DBLP \\
	\colrule
		ChebyNet	&31.34$\pm$0.12   &57.44$\pm$0.27   &26.84$\pm$0.62                   \\
		GCN   &54.95$\pm$0.14  &56.52$\pm$0.09     &38.5$\pm$0.97         \\
		GAT   &56.54$\pm$0.68   &56.41$\pm$0.07       &77.1$\pm$1.86         \\
		APPNP   &57.88$\pm$0.32    &56.61$\pm$0.07       &79.34$\pm$0.23         \\
\colrule
		AdaGCN   &25.01$\pm$0.00      &37.03$\pm$0.00     &9.60$\pm$0.00              \\
		BGNN     &66.48$\pm$0.22    &66.18$\pm$0.11         &87.2$\pm$0.60         \\
\colrule
		Sim$\_$Stacking  &53.32$\pm$0.15             &56.11$\pm$0.08   &71.49$\pm$0.31         \\
		SStaGCN (Mean)    &72.23$\pm$0.04    &66.74$\pm$0.21     &82.13$\pm$0.38         \\
		SStaGCN (Attention)  &72.36$\pm$0.09      &77.62$\pm$0.10        &82.68$\pm$0.12         \\
		SStaGCN (Voting)     &\textbf{75.45$\pm$0.82}   &\textbf{87.84$\pm$0.04}    &\textbf{92.64$\pm$0.06}         \\
\botrule
	\end{tabular}}
	\label{f1 score_het}
\end{table}

\begin{table}[th]
	\tbl {p-values of the paired t-test of SStaGCN (Voting) with competitors on $6$ different datasets (Cora,  Citeseer,  Pubmed, House$\_$class, VK$\_$class, and  DBLP).}
	{\begin{tabular}{@{}lcccccc@{}} \toprule

		Models   &Cora   &CiteSeer    &Pubmed    &House$\_$class  &   VK$\_$class &  DBLP \\
\colrule
		ChebyNet   &2.59e-06      &2.77e-08   &1.17e-06     &3.51e-10      &1.11e-11     &6.97e-09 \\
		GCN   &4.19e-16      &5.15e-17   &1.84e-15     &2.36e-08      &2.25e-11     &2.48e-07 \\
		GAT   &2.10e-19      &1.54e-20   &8.84e-19     &3.35e-08      &1.38e-09     &4.95e-06 \\
		APPNP&6.86e-42      &7.12e-42   &6.25e-41     &7.05e-15      &1.39e-21    &2.26e-09       \\
		AdaGCN   &1.82e-19      &1.23e-20  &8.42e-19    &3.05e-38      &7.94e-43     &1.69e-35 \\
		BGNN   &1.02e-24      &2.33e-21   &4.74e-22     &6.54e-07      &1.07e-08     &0.11e-03 \\
		Sim$\_$Stacking   &5.61e-12    &3.79e-15   &2.52e-09   &4.56e-08      &6.07e-11     &1.07e-06 \\
\botrule
	\end{tabular}}
	\label{t test}
\end{table}
The results of the comparative evaluation for node classification are summarized in Tables 3-11. In these tables, SStaGCN (Mean) refers to the use of the mean aggregation mechanism in the second layer of SStaGCN, while SStaGCN (Attention) and SStaGCN (Voting) correspond to the attention and voting mechanisms, respectively. We report the accuracy, macro F1-score, and training time on the test set for the proposed SStaGCN model and other methods. The experimental results demonstrate a significant improvement of the SStaGCN model over the baselines. Specifically, for the three public citation networks, SStaGCN (Voting) achieves an improvement in accuracy (and F1-score) of approximately $8\%$ ($8\%$), $12\%$ ($17\%$), and $13\%$ ($12\%$) for the Cora, CiteSeer, and Pubmed datasets, respectively. For the heterogeneous tabular datasets, SStaGCN (Voting) shows an improvement in accuracy (and F1-score) of about $9\%$ ($9\%$), $21\%$ ($21\%$), and $6\%$ ($5\%$) for the House$\_$class, VK$\_$class, and DBLP datasets, respectively.

AdaGCN performs poorly on the heterogeneous tabular datasets, possibly because AdaGCN is designed for deeper GCN architectures, which may mix node features from different clusters as the GCN layers deepen. SStaGCN, by contrast, enhances the performance of GCN, providing better results across distinct types of graph-structured data. Additionally, the paired t-test results in Table 7 indicate that the proposed SStaGCN model significantly outperforms the simplified stacking and other GCN models.

This impressive improvement can be explained as follows:

\begin{itemlist}
	\item  The stacking and aggregation steps in SStaGCN provide a dimensionality reduction effect, making the graph data more discernible. For example, in the Cora dataset, the feature size reduces from  $2708\times1433$  to $2708\times7 $ after stacking and aggregation. This results in a relatively smaller $d_1$, as discussed in Remark 1, significantly enhancing both the predictive power and computational efficiency of the subsequent graph convolution model.
	
\item In the aggregation step of the SStaGCN model, the mean and attention mechanisms somewhat disrupt the pre-classification results, making them less suitable for feature extraction, whereas the voting mechanism preserves these results. Consequently, experimental results demonstrate that SStaGCN (Voting) is more effective across the six datasets.

\item Simplified stacking can effectively extract useful attributes but overlooks graph structure information, while GCN models are limited in feature extraction. Therefore, the SStaGCN model combines the strengths of simplified stacking and GCN, achieving both higher classification accuracy and reduced computational cost.
\end{itemlist}

Tables 8 and 9 present a comparison of training times between SStaGCN and other methods. BGNN achieves the fastest runtime on the three citation networks, followed closely by our SStaGCN method. However, SStaGCN runs faster than other methods on the heterogeneous tabular datasets, with the exception of the DBLP dataset. We attribute this to the extra computation time required for the stacking and aggregation steps, which ultimately enhance efficiency when feeding the feature outputs into the GCN model.

\begin{table}[th]
	\tbl {Average training time(s) on $3$ citation networks by computing  the $95\%$ confidence interval via bootstrap.}
	{\begin{tabular}{@{}lccc@{}} \toprule
		Method     &Cora  &CiteSeer    &Pubmed\\
\colrule
		ChebyNet	&22.74$\pm$1.24             &30.87$\pm$1.21      &124.99$\pm$1.77         \\
		GCN   &13.41$\pm$0.16   &99.21$\pm$0.98      &55.61$\pm$0.73         \\
		GAT   &20.98$\pm$0.46     &30.74$\pm$1.40     &126.33$\pm$1.80         \\
		APPNP   &203.75$\pm$0.15     &55.40$\pm$0.40           &457.62$\pm$12.77         \\
\colrule
		AdaGCN   &772.26$\pm$83.56    &2129.02$\pm$148.97         &2098.10$\pm$275.88         \\
		BGNN     &\textbf{1.33$\pm$0.00}  & \textbf{2.40$\pm$0.00}       & \textbf{2.54$\pm$0.00}        \\
\colrule
		Sim$\_$Stacking  &11.9$\pm$0.08  &27.9$\pm$0.24    & 79.1$\pm$1.40 \\
		SStaGCN (Mean)	& 10.9$\pm$0.13 	& 17.2$\pm$0.13  	&89.6$\pm$2.20\\
		SStaGCN (Attention)	&11.2$\pm$0.24	&17.6$\pm$0.41	&87.6$\pm$0.96\\
		SStaGCN (Voting)	&16.2$\pm$0.19	&29.6$\pm$1.40	&13.1$\pm$2.61\\
\botrule
	\end{tabular}}
	\label{time_cit}
\end{table}

\begin{table}[th]
	\tbl {Average training time(s)  on heterogeneous tabular datasets  by calculating the $95\%$ confidence interval via bootstrap.}
	{\begin{tabular}{@{}lccc@{}} \toprule
		Method     &House$\_$class  &   VK$\_$class &  DBLP \\
\colrule
		ChebyNet	&833.82$\pm$0.00             &1394.68$\pm$0.00       &8890.74$\pm$0.00         \\
		GCN   &46.06$\pm$0.80   &120.1$\pm$3.35    &268.5$\pm$5.35         \\
		GAT   &197.6$\pm$3.08   &410.9$\pm$9.95       &205.3$\pm$2.00         \\
		APPNP   &129.7$\pm$3.26     &383.8$\pm$12.02      &176.8$\pm$5.58         \\
\colrule
		AdaGCN   &607.31$\pm$0.00     &511.41$\pm$0.00                  &590.90$\pm$0.00              \\
		BGNN     &26.37$\pm$1.65     &93.47$\pm$6.23        &\textbf{50.25$\pm$3.04}         \\
\colrule
		Sim$\_$Stacking  &\textbf{16.41$\pm$0.36}            & \textbf{43.58$\pm$2.78}       &380.8$\pm$12.82        \\
		SStaGCN (Mean)	&52.94$\pm$0.51   &132.9$\pm$1.61       &188.7$\pm$0.54         \\
		SStaGCN (Attention)&57.38$\pm$1.75             &133.2$\pm$1.11        &246.2$\pm$0.37        \\
		SStaGCN (Voting)	&59.69$\pm$0.82       &154.1$\pm$1.02         &310.7$\pm$0.49     \\
\botrule
	\end{tabular}}
	\label{time_het}
\end{table}
To demonstrate the effect of stacking and aggregation steps in the proposed model, we provide a t-SNE visualization \cite{Maaten2008VisualizingDU} in Figs. \ref{hom_tsne_cls_fea} and \ref{het_tsne_cls_fea}. These figures show that the combination of stacking and aggregation effectively extracts features, enhancing the discriminative power of the graph data.
\begin{figure}[!htpb]
	\centering
	\subfigure[GCN]{\label{gcn:raw}\includegraphics[width=0.41\textwidth]{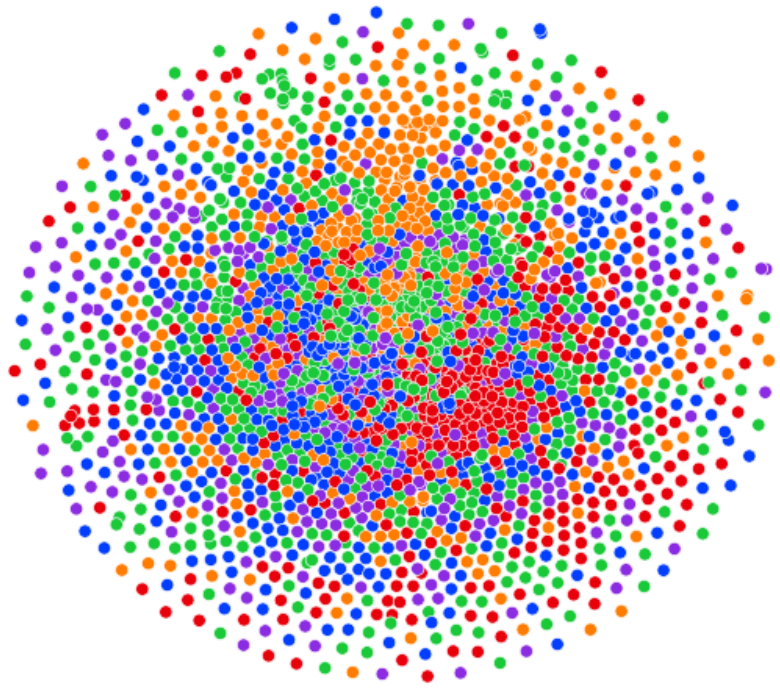}}
	\mbox{\hspace{0.1cm}}
	\subfigure[SStaGCN]{\label{sstagcn:stack}\includegraphics[width=0.46\textwidth]{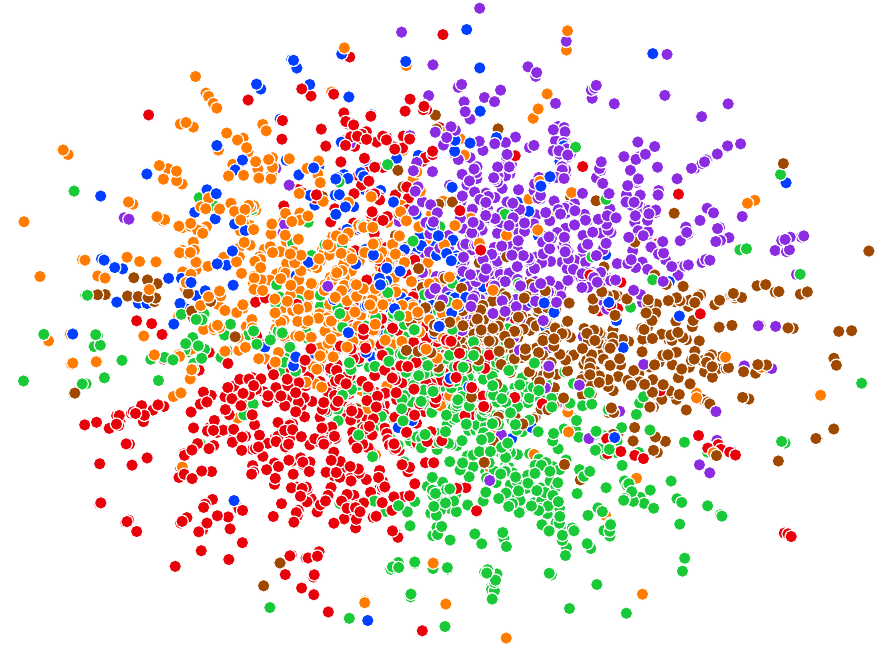}}
	\mbox{\hspace{0.1cm}}
	\vspace{-0.3cm}
	\caption{Visualization of  classification features by the GCN (left) and the features  after conducting stacking and aggregation steps in the SStaGCN model (right) on CiteSeer dataset, node colors denote classes.}
	\label{hom_tsne_cls_fea}
\end{figure}

\begin{figure}[!htpb]
	\centering
	\subfigure[GCN]{\label{hetgcn:raw}\includegraphics[width=0.42\textwidth]{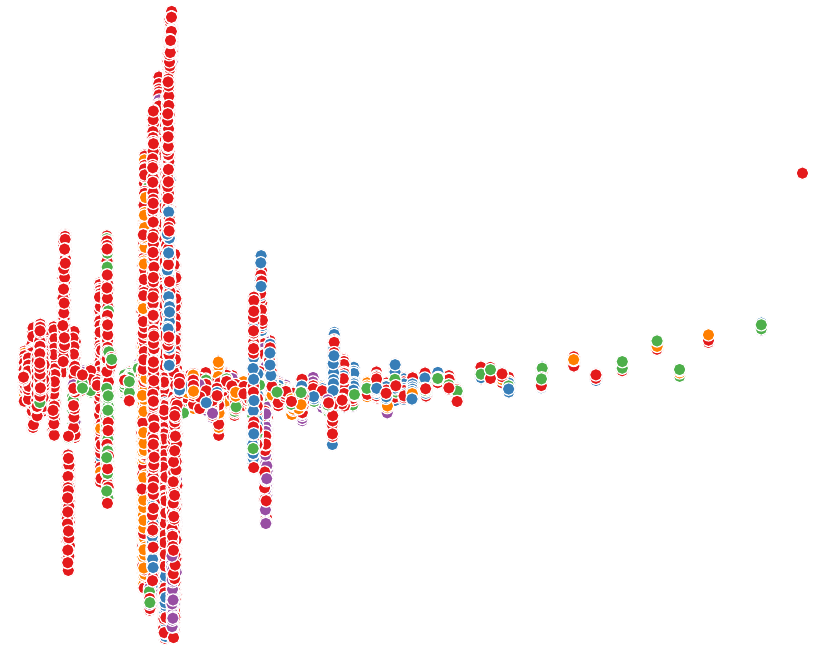}}
	\mbox{\hspace{0.1cm}}
	\subfigure[SStaGCN]{\label{hetsstagcn:stack}\includegraphics[width=0.43\textwidth]{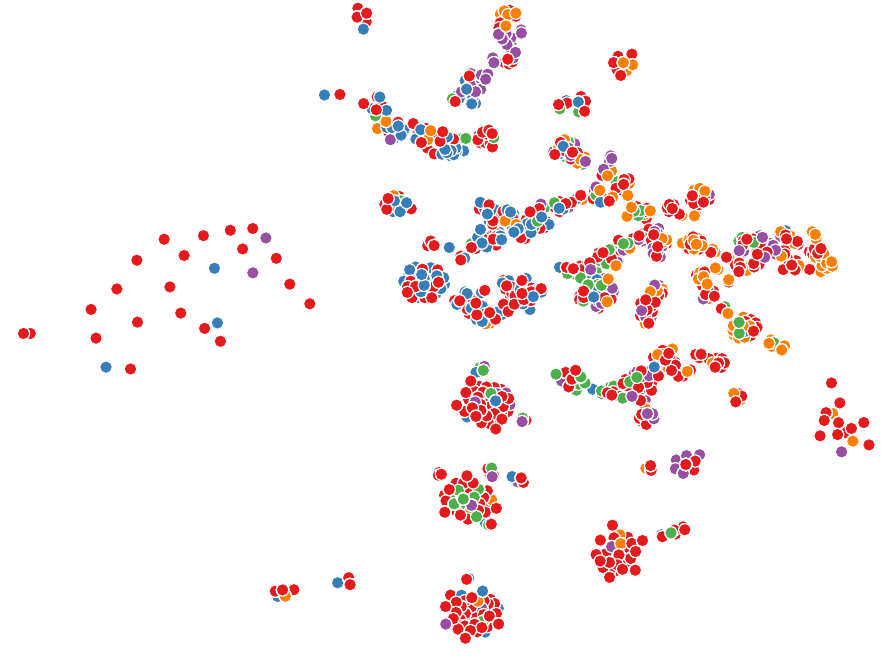}}
	\mbox{\hspace{0.1cm}}
	\vspace{-0.3cm}
	\caption{Visualization of  classification features by the GCN (left) and the features  after conducting stacking and aggregation steps in the SStaGCN model (right) on DBLP dataset, node colors denote classes.}
	\label{het_tsne_cls_fea}
\end{figure}
Table 10 suggests that using all seven classifiers is unnecessary. For example, on the Cora dataset, the combination of KNN, Random Forest, and Naive Bayes achieves the highest accuracy with minimal computational cost (only 3 seconds). This indicates that each of the seven classifiers has unique strengths suited to different tasks. However, selecting the optimal combination of base models remains an open area for theoretical analysis.
\begin{table}[th]\footnotesize
	\tbl {Accuracy and training time(s) on Cora dataset by combinations of different classifiers based on SStaGCN model.}
	{\begin{tabular}{@{}ccccccccc@{}} \toprule
		KNN &Random Forest &Naive Bayes &Decision Tree &GBDT   &Adaboost   &SVC    &Accuracy      &Training Time\\
	\colrule
		&$\surd$ &$\surd$ & & & & &91.2 &13.90 \\
		$\surd$  &$\surd$ &$\surd$ & & & & &\textbf{93.6} &16.60\\
		& & & &$\surd$ &$\surd$ & &84.2 &567.9\\
		& &$\surd$ &$\surd$ & & &$\surd$ &92.9 &144.7\\
		&$\surd$ &$\surd$ & & & &$\surd$ &93.1 &149.5\\
		&$\surd$ &$\surd$ &$\surd$ & & & &92.8 &15.90\\
		$\surd$  &$\surd$ &$\surd$ &$\surd$ & & & &93.4 &18.80\\
		$\surd$  &$\surd$ &$\surd$ &$\surd$ &$\surd$ & & &92.9 &568.3\\
		$\surd$  &$\surd$ &$\surd$ &$\surd$ &$\surd$ &$\surd$ & &92.9 &570.7\\
		$\surd$  &$\surd$ &$\surd$ &$\surd$ &$\surd$ &$\surd$ &$\surd$ &92.8 &708.5\\
\botrule
	\end{tabular}}
\label{classifier}
\end{table}
To further illustrate the impact of simplified stacking on the over-smoothing problem, we conduct an additional experiment on this topic. In Fig. \ref{over-smoothing_layer_GCN}, we observe that a conventional GCN tends to blend node features from different clusters as the number of convolutional layers increases. However, as shown in Fig. \ref{over-smoothing_layer_SStaGCN} and Table 11\footnote{The number of layers excludes those included in the stacking and aggregation parts of SStaGCN, which consist of only two layers.}, the proposed SStaGCN effectively mitigates the over-smoothing phenomenon and enhances accuracy.
\begin{figure*}[!htpb]
	\centering
	\subfigure[2-layer]{\includegraphics[width=0.28\textwidth]{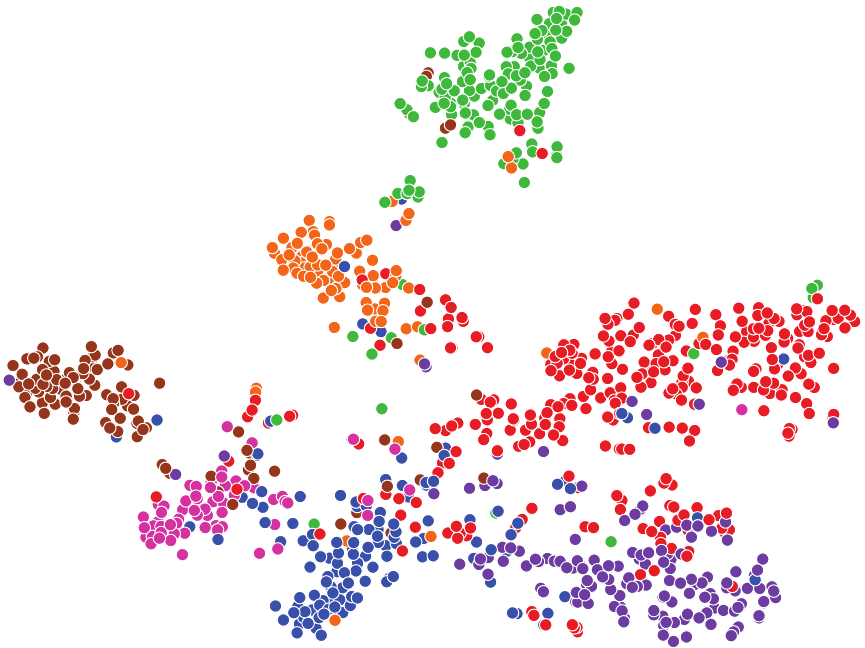}}
	\mbox{\hspace{0.1cm}}
	\subfigure[3-layer]{\includegraphics[width=0.28\textwidth]{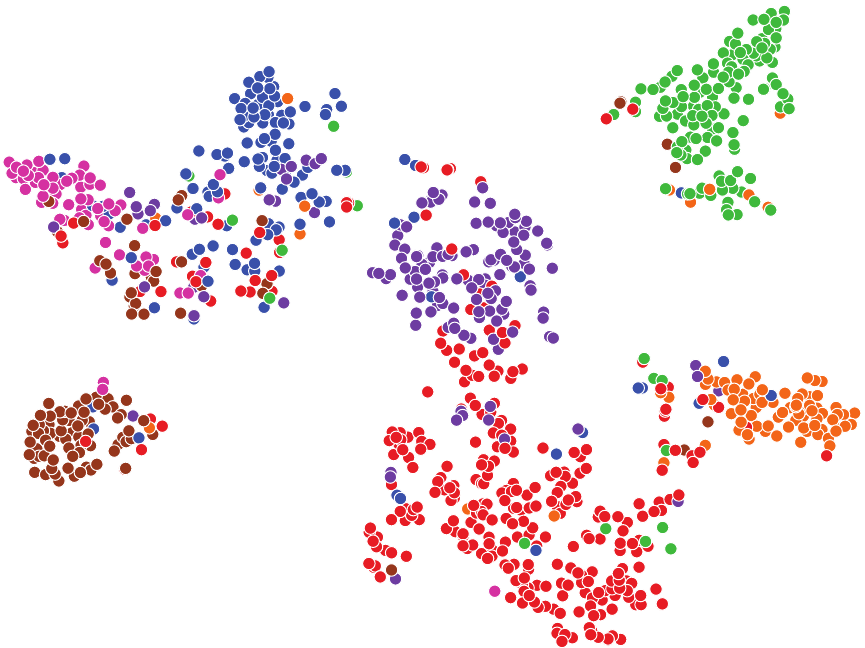}}
	\mbox{\hspace{0.1cm}}
	\subfigure[4-layer]{\includegraphics[width=0.28\textwidth]{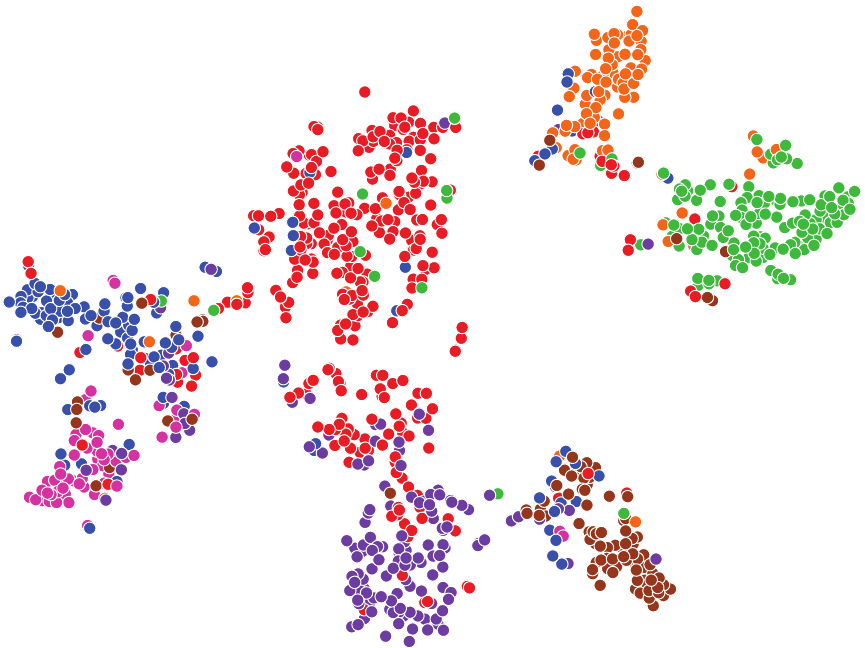}}
	\mbox{\hspace{0.1cm}}
	\subfigure[5-layer]{\includegraphics[width=0.28\textwidth]{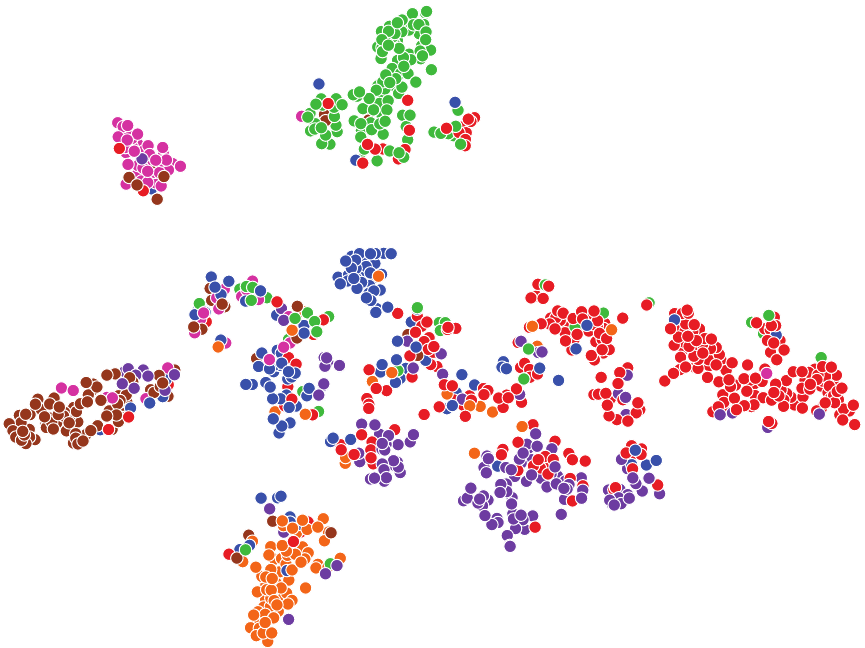}}
	\mbox{\hspace{0.1cm}}
	\subfigure[6-layer]{\includegraphics[width=0.28\textwidth]{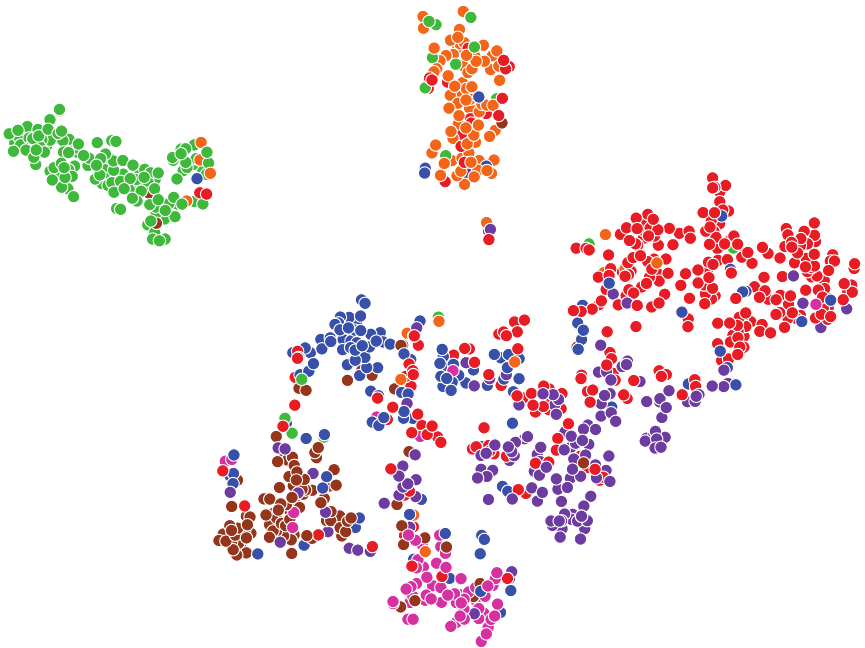}}
	\mbox{\hspace{0.1cm}}
	\subfigure[7-layer]{\includegraphics[width=0.28\textwidth]{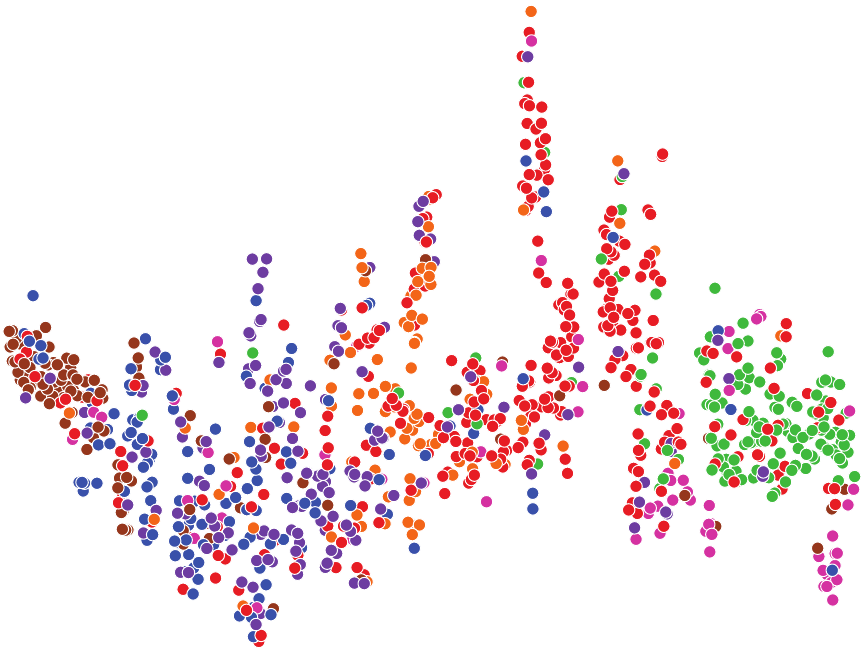}}
	\mbox{\hspace{0.1cm}}
	\vspace{-0.3cm}
	\caption{Visualization  of  final classification features via GCN  on  Cora  dataset with $2$, $3$, $4$, $5$, $6$, $7$ layers, node colors denote classes.}
	\label{over-smoothing_layer_GCN}
\end{figure*}

\begin{figure*}[!htpb]
	\centering
	\subfigure[2-layer]{\includegraphics[width=0.28\textwidth]{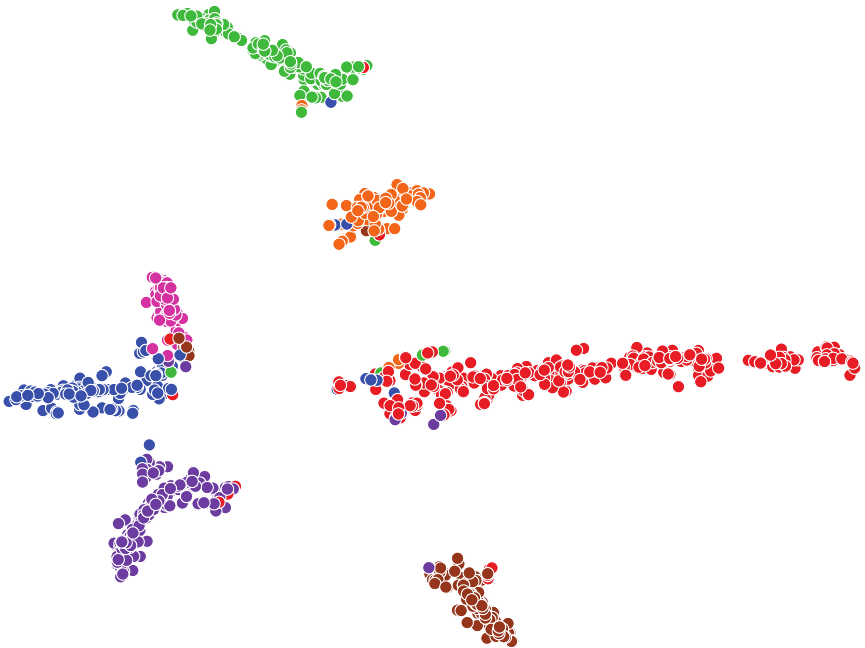}}
	\mbox{\hspace{0.1cm}}
	\subfigure[3-layer]{\includegraphics[width=0.28\textwidth]{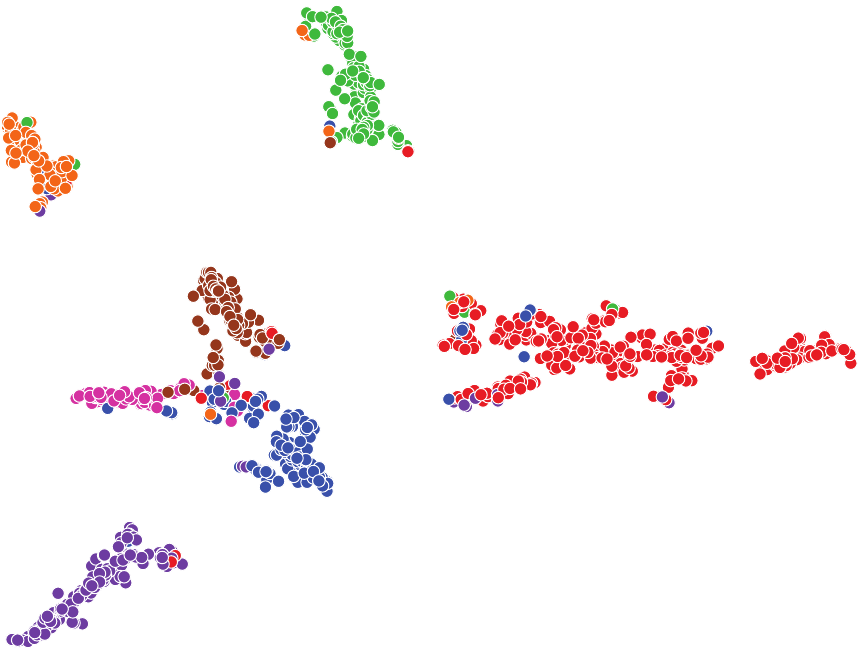}}
	\mbox{\hspace{0.1cm}}
	\subfigure[4-layer]{\includegraphics[width=0.28\textwidth]{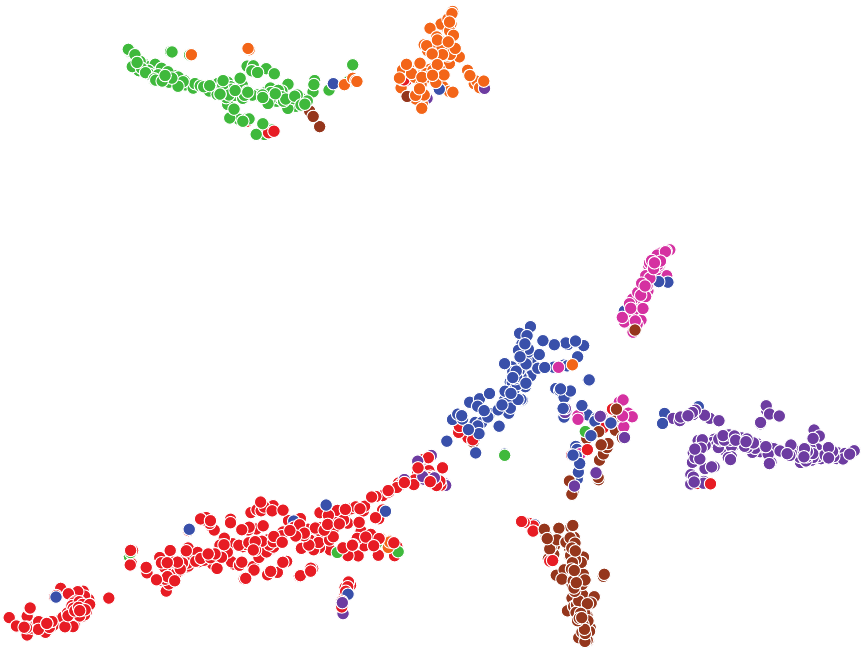}}
	\mbox{\hspace{0.1cm}}
	\subfigure[5-layer]{\includegraphics[width=0.28\textwidth]{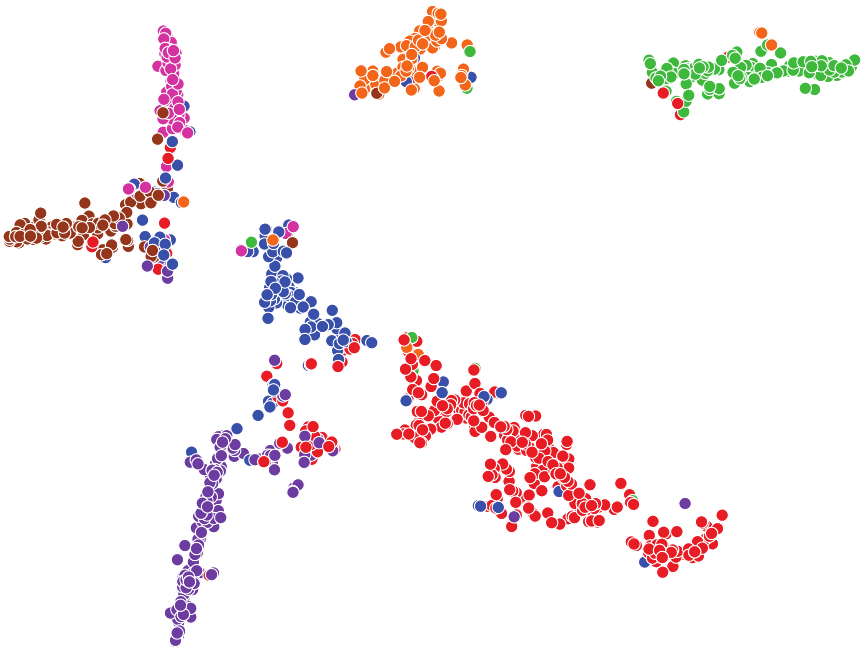}}
	\mbox{\hspace{0.1cm}}
	\subfigure[6-layer]{\includegraphics[width=0.28\textwidth]{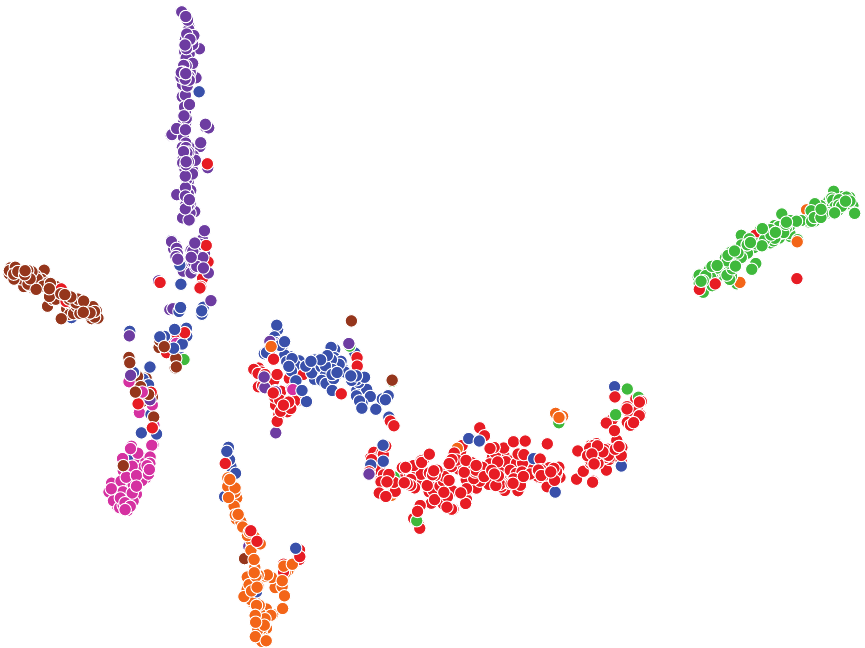}}
	\mbox{\hspace{0.1cm}}
	\subfigure[7-layer]{\includegraphics[width=0.28\textwidth]{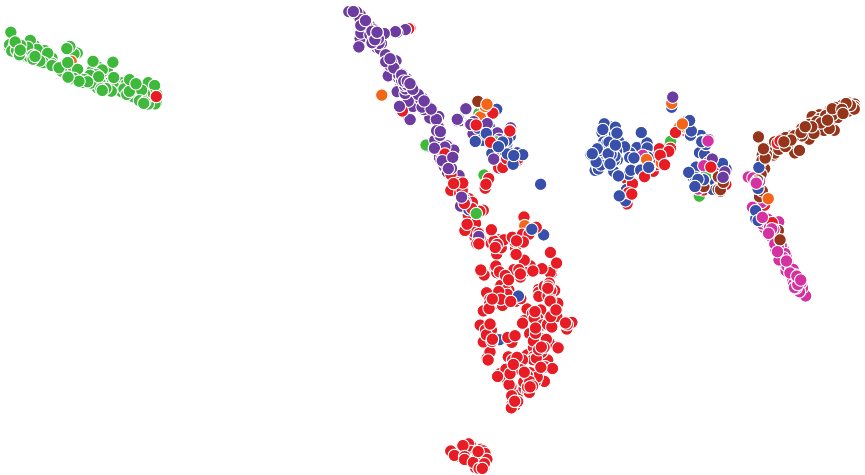}}
	\mbox{\hspace{0.1cm}}
	\vspace{-0.3cm}
	\caption{Visualization  of  final classification features via SStaGCN  on  Cora  dataset with $2$, $3$, $4$, $5$, $6$, $7$ layers, node colors denote classes.}
	\label{over-smoothing_layer_SStaGCN}
\end{figure*}

\begin{table}[th]
	\tbl {Accuracy comparison between GCN and SStaGCN models on Cora dataset using  distinct number of layers.}
	{\begin{tabular}{@{}ccccccc@{}} \toprule
		Method &2-layer    &3-layer       &4-layer     &5-layer   &6-layer   &7-layer  \\
\colrule
		GCN &80.5 &80.4 &75.8 &71.9 &72.6 &60.8 \\
		SStaGCN  &93.3 &88.8 &87.5 &86.4 &84.8 &84.3 \\
\botrule
	\end{tabular}}
	\label{Oversmooth_SStaGCN}
\end{table}
Overall, these experiments demonstrate the superiority of SStaGCN model over competitors.

\subsection{Visualization}
\begin{figure}[!htpb]
	\centering
	\subfigure[GCN]{\label{gcn:final}\includegraphics[width=0.23\textwidth]{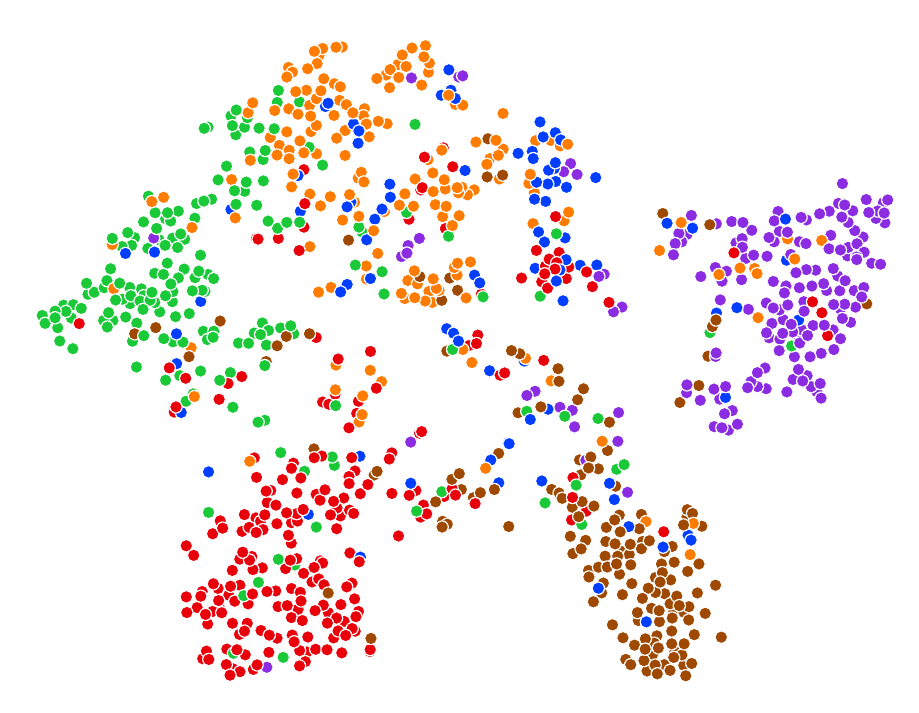}}
	\mbox{\hspace{0.1cm}}
	\subfigure[AdaGCN]{\label{sstagcn:final}\includegraphics[width=0.22\textwidth]{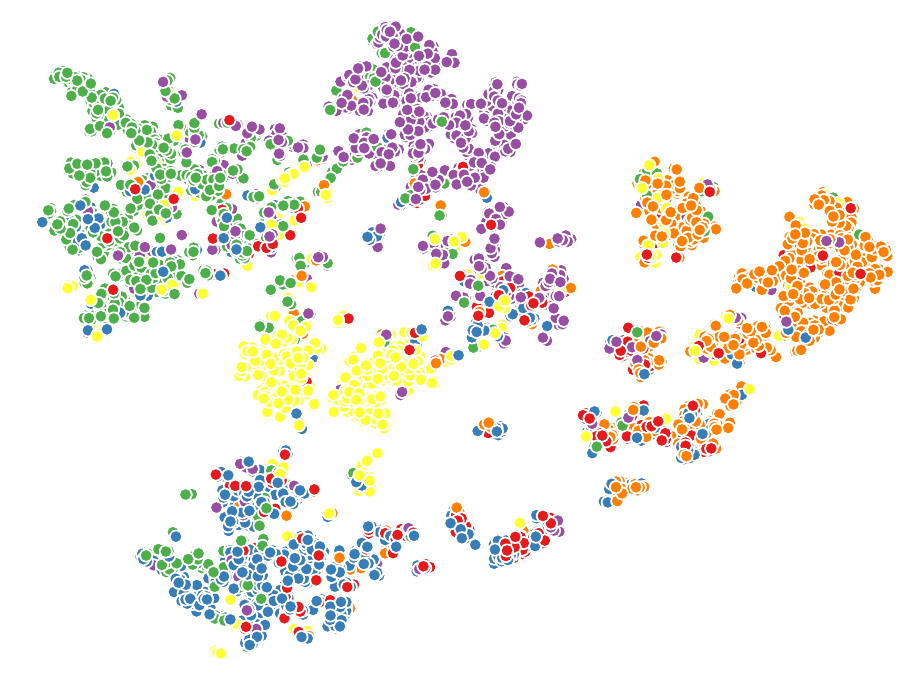}}
	\mbox{\hspace{0.1cm}}
	\subfigure[BGNN]{\label{gcn:final}\includegraphics[width=0.23\textwidth]{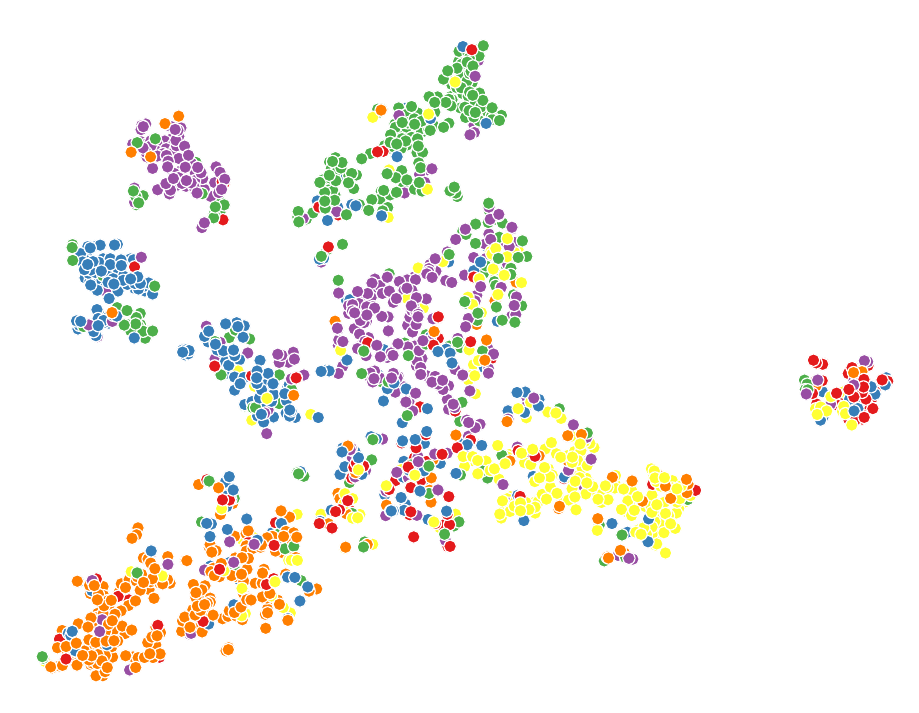}}
	\mbox{\hspace{0.1cm}}
	\subfigure[SStaGCN]{\label{sstagcn:final}\includegraphics[width=0.22\textwidth]{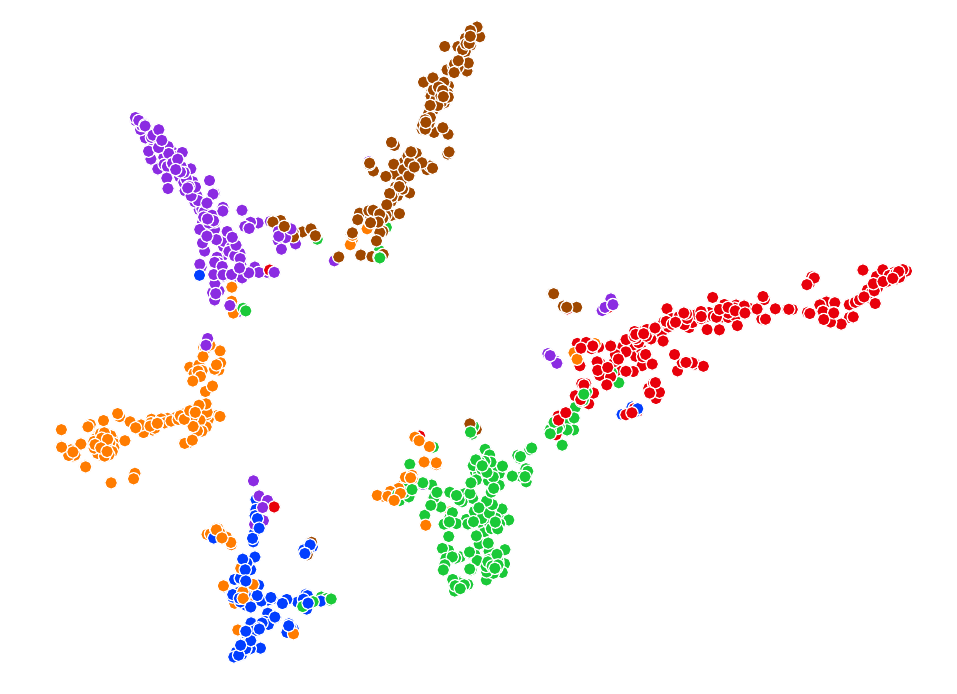}}
	\mbox{\hspace{0.1cm}}
	\vspace{-0.3cm}
	\caption{Visualization  of  final classification features via (a). GCN , (b). AdaGCN ,  (c). BGNN , and (d). SStaGCN model on Citeseer dataset,   node colors denote classes.}
	\label{hom_tsne_fincls_fea}
\end{figure}

\begin{figure}[!htpb]
	\centering
	\subfigure[GCN]{\label{gcn:final}\includegraphics[width=0.23\textwidth]{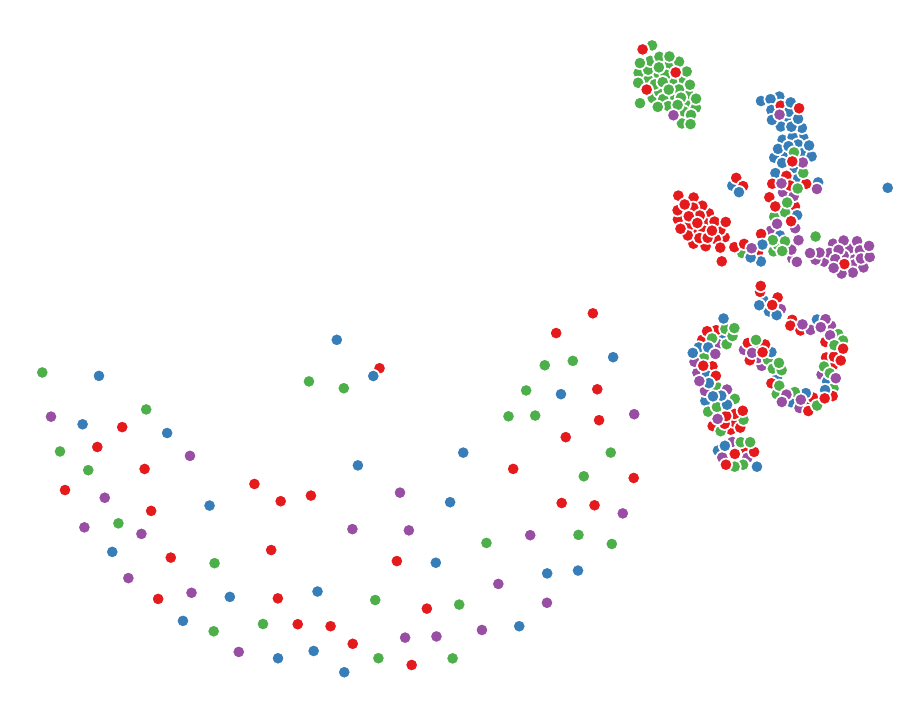}}
	\mbox{\hspace{0.1cm}}
	\subfigure[AdaGCN]{\label{sstagcn:final}\includegraphics[width=0.23\textwidth]{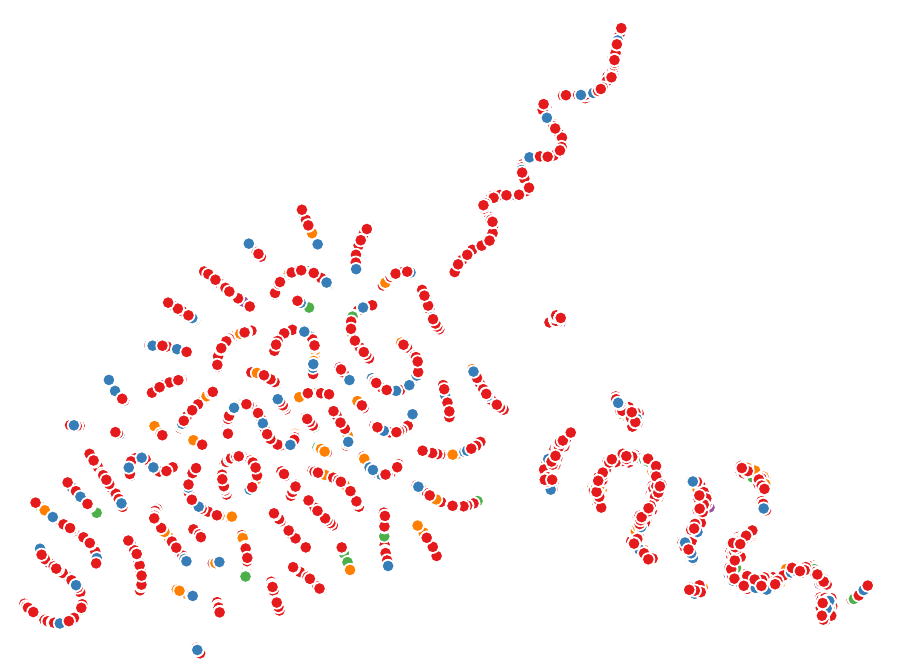}}
	\mbox{\hspace{0.1cm}}
	\subfigure[BGNN]{\label{sstagcn:final}\includegraphics[width=0.22\textwidth]{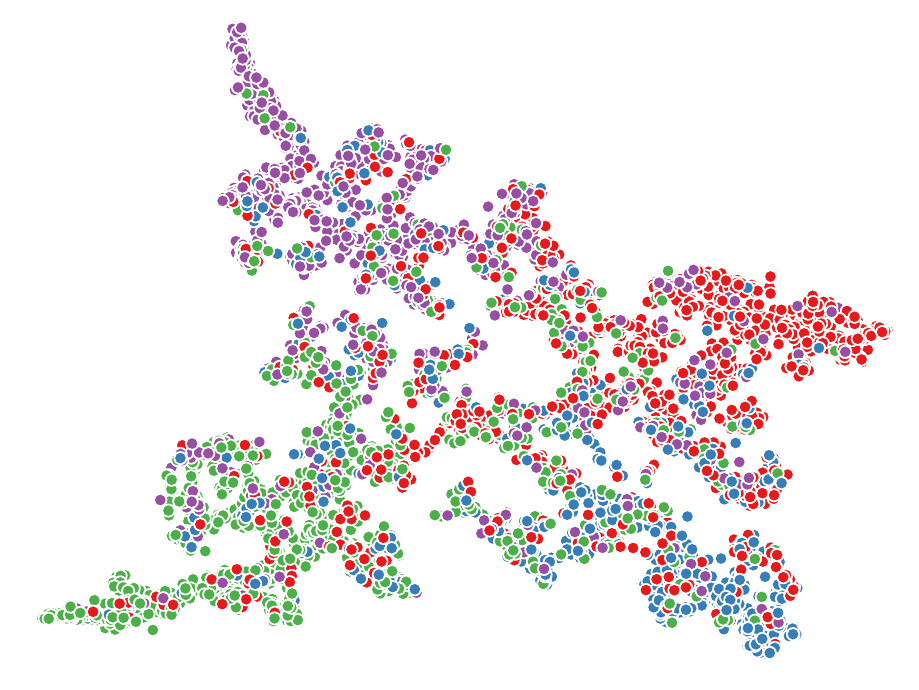}}
	\mbox{\hspace{0.1cm}}
	\subfigure[SStaGCN]{\label{sstagcn:final}\includegraphics[width=0.22\textwidth]{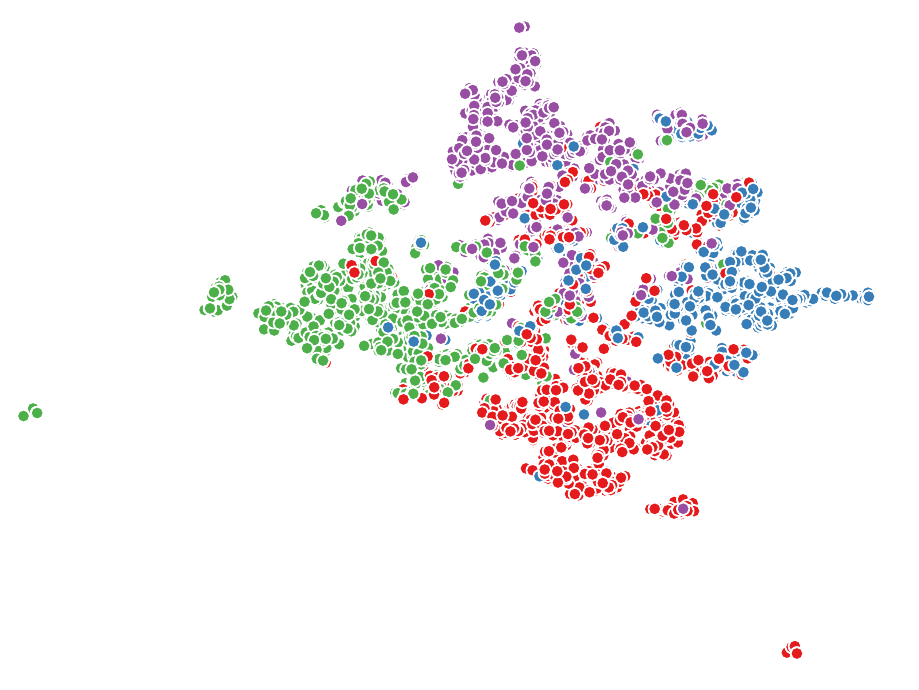}}
	\mbox{\hspace{0.1cm}}
	\vspace{-0.3cm}
	\caption{Visualization  of  final classification features via (a). GCN , (b). AdaGCN , (c). BGNN,  and (d). SStaGCN model on  DBLP dataset,  node colors denote classes.}
	\label{het_tsne_fincls_fea}
\end{figure}
To further illustrate the performance of SStaGCN, we plot the final classification features of GCN, AdaGCN, BGNN, and our SStaGCN. Fig. \ref{hom_tsne_fincls_fea} (for CiteSeer) and Fig. \ref{het_tsne_fincls_fea} (for DBLP) display the final classification features of each method on these datasets. As shown in these figures, the proposed SStaGCN misclassifies relatively fewer points, while GCN, AdaGCN, and BGNN result in more misclassified classes.

\section{Conclusion}\label{concl}
Traditional GCNs often face the over-smoothing problem. In this work, we propose a novel GCN architecture, SStaGCN, which leverages stacking and aggregation techniques to capture pre-classified data features, followed by GCN for predictions on distinct graph-structured data. SStaGCN effectively explores and utilizes features for heterogeneous graph data in a stacked manner. By incorporating classical machine learning methods, we design a GCN model that provides a versatile framework for handling diverse types of graph-structured data, offering new insights into understanding GCNs. Extensive experiments demonstrate that the proposed model outperforms several state-of-the-art competitors in terms of accuracy, F1-score, and training time. The framework presented here could also be extended to regression tasks. Additionally, we believe this method can address various types of heterogeneous graph data beyond tabular data. A promising future research direction is to investigate deeper GCNs within our framework, as suggested by \cite{Sun2019AdaGCNAG}. The source code of SStaGCN will be issued soon \footnote{https://github.com/dragon0916/SStaGCN.}.
\section*{Appendix}\label{section:appe}
In this part, we provide a detailed proof of Theorem \ref{mainthm}. Before proceeding with the proof, we introduce several supporting lemmas. First, we present the contraction inequality for Rademacher complexity in vector form.
\begin{lemma}
	\cite{2016Maurer}
	Let ${\cal X}$ be any set, $({\bf x}_1,\cdots, {\bf x}_N)\in {\cal X}$ and ${\cal F}$ be a class of functions $f:{\cal X}\to {\mathbb R}^K$ ad let $\tau_i: {\mathbb R}^K\to {\mathbb R}$ have Lipschitz constant $M$. Then
	$${\mathbb E}\sup_{f\in {\cal F}}\sum^N_{i=1}\sigma_i \tau_i f({\bf x}_i)\leq \sqrt{2} M{\mathbb E}\sup_{f\in {\cal F}} \sum^N_{i=1}\sum^K_{k=1}\sigma_{ik} f_k({\bf x}_i),$$
	where $\sigma_{ik}$ is an independent doubly indexed Rademacher sequence and $f_k({\bf x}_i)$ is the $k-$th component of $f({\bf x}_i)$.	
	\label{vector_contr}
\end{lemma}

\begin{lemma}
	\cite{2001Rademacher}
	Consider a loss function $L: {\cal X}\times {\cal Y}\to {\mathbb R}^+$. Denote  ${\cal H}=\{L (y, f(\cdot)), f\in {\cal F}\}$, and let $({\bf x}_i, y_i)^N_{i=1}$ be independently selected according to the probability measure $\mathbb P$. Then for any $0<\delta<1$, with probability at least $1-\delta$,
	$${\cal E}(f)\leq {\cal E}_N(f)+2{ \cal {\widehat R}}({\cal H})+\sqrt{\frac{2\log(2/\delta)}{N}}, ~~\forall f\in {\cal H}.$$
	\label{rademacher_ineq}
\end{lemma}
We first give a lemma which plays essential role in the proof of Theorem \ref{mainthm}.

\begin{lemma}
	Let $q=\max\{N(v)\}$ for each node $v\in \Omega$, then
	$$\max_{v} \|\sum_{j\in N(v)} {\hat A}_{vj}{\bf x}_j\|^2_2\leq  R^2 q.$$
	\label{critial lemma}
\end{lemma}
\begin{proof}
~Denote ${\hat A}_v\in {\mathbb R}^{q_v\times q_v}$ as the sub-matrix of $\hat A$ whose row and column indices belong to the set $j\in N(v)$. Hence the size of ${\hat A}_v$ depends on the node $v$. Obviously, $q=\max{q_v}$ for $v\in \Omega$. Let ${\tilde X}_v =( {\bf x}^T_1, \cdots,{\bf x}^T_q)^T \in {\mathbb R}^{q\times d}$ be the feature matrix of the nodes in ${\cal G}_v$ (subgraph of $\cal G$). Hence,
$$
\max_{v}\Big\|\sum_{j\in N(v)} {\hat A}_{vj}{\bf x}_j\Big\|^2_2= \max_{v}\|{\hat A}_{v\cdot}{\tilde X}_v\|^2_2\leq \max_{v}\|{\hat A}_{v\cdot}\|^2_2 \|{\tilde X}_v\|^2_2,$$
where $\hat A_{v\cdot}$ is the $v-\!$th row of the matrix $\hat A$  with column index $j$ belong to the set $N(v)$. Notice that
$$\|{\tilde X}_v\|_2= \sup_{\|t\|_2=1} \|{\tilde X}_v t\|_2 \leq \sqrt{\sum^{q_v}_{s=1}\|{\bf x}_s\|^2_2}\leq R\sqrt{q_v}\leq R\sqrt{q},$$
and $\|\hat A\|_2\leq 1$. Therefore,
$$\max_{v}\Big\|\sum_{j\in N(v)} {\hat A}_{vj}x_j\Big\|^2_2\leq \|\hat A\|^2_2\|{\tilde X}_v\|^2_2\leq R^2 q.$$
\end{proof}

Now we are in position to give the proof of Theorem \ref{mainthm}.

\begin{proof} [Proof of Theorem \ref{mainthm}]
To allow a slight abuse of notations, we will use $X_j$ to denote $\tilde {X}_j$ due to the explanation on page 3. Denote $h(W^{(0)})= {\rm ReLU}$$({\hat A}XW^{(0)})$, $f(W^{(1)})$$ =$ ${\rm softmax}$ $({\hat h}(W^{(0)}) W^{(1)})$$=$ $ (f_1(W^{(1)}), \cdots, f_m(W^{(1)}))^T$ with ${\hat h}(W^{(0)})={\hat A} h(W^{(0)})$.
Applying Proposition 4 in \cite{Gao2017OnTP} to the case $\lambda=1$, we can attain the Lipschitz constant for standard softmax function is $M=1$.
Let $\hat A_{i\cdot}$ stands for the $i-$th row of the matrix $\hat A$, for  function set
\begin{align*}
&{\cal F}_{B_1, B_2}=\{f_i = {\rm softmax}({\hat A}_{i\cdot}{\rm ReLU} ({\hat A}X W^{(0)})W^{(1)}),i=1,\cdots,m, \\
&\|W^{(0)}\|_F\leq B_1, \|W^{(1)}\|_F\leq B_2\},
\end{align*}
the empirical Rademacher complexity is defined as
$${\cal {\hat R}} ({{\cal F}}_{B_1,B_2}) = {\mathbb E}_\sigma \Big[\frac1m\sup_{\stackrel{\|W^{(0)}\|_F\leq B_1}{\|W^{(1)}\|_F\leq B_2}} \sum^m_{i=1} \sigma_i f_i(W^{(1)})\Big],$$
where $\{\sigma_i\}^m_{i=1}$ is an i.i.d. family of Rademacher variables independent of ${\bf x}_i$. By the contraction property of Rademacher complexity,
$${\cal {\hat R}}({\cal H}) \leq \alpha_L{\cal {\hat R}}({{\cal F}}_{B_1, B_2}),$$
and notice Lemma \ref{rademacher_ineq}, we only need to bound ${\cal \hat R} ({\cal F}_{B_1,B_2})$.
Therefore, we have the following estimate by utilizing Lemma \ref{vector_contr}.
\begin{align*}
{\cal \hat R} ({\cal F}_{B_1,B_2}) &={\mathbb E}_\sigma \Big[\frac1m\sup_{\stackrel{\|W^{(0)}\|_F\leq B_1}{\|W^{(1)}\|_F\leq B_2}} \sum^m_{i=1} \sigma_i f_i(W^{(1)})\Big] \\
&\leq\frac{\sqrt{2}}{m}  {\mathbb E}_\sigma\Big[\sup_{\stackrel{\|W^{(0)}\|_F\leq B_1}{\|W^{(1)}\|_F\leq B_2}}\sum^m_{i=1} \sum^K_{k=1} \sigma_{ik}{\hat h}_i(W^{(0)}){\bf w}^{(1)}_k\Big],
\end{align*}
where $W^{(1)}= ({\bf w}^{(1)}_1,\cdots, {\bf w}^{(1)}_K)$, notice the property of inner product, the above estimate can be further bounded as
\begin{align*}
{\cal \hat R} ({\cal F}_{B_1,B_2})
&\leq \frac{\sqrt{2}}{m} {\mathbb E}_\sigma \Big[\sup_{\stackrel{\|W^{(0)}\|_F\leq B_1}{\|W^{(1)}\|_F\leq B_2}} \sum^K_{k=1} \max_{k\in [K]}\|{\bf w}^{(1)}_k\|_2 \big\| \sum^m_{i=1} \sigma_{ik} {\hat h}_i(W^{(0)})\big\|_2\Big],\\
&\leq \frac{\sqrt{2}}{m} {\mathbb E}_\sigma \Big[\sup_{\stackrel{\|W^{(0)}\|_F\leq B_1}{\|W^{(1)}\|_F\leq B_2}} \|W^{(1)}\|_F \sum^K_{k=1} \big\| \sum^m_{i=1} \sigma_{ik} {\hat h}_i(W^{(0)})\big\|_2\Big],\\
&\leq \frac{\sqrt{2}B_2}{m} {\mathbb E}_\sigma \Big[\sup_{\|W^{(0)}\|_F\leq B_1} \sum^K_{k=1} \big\| \sum^m_{i=1} \sigma_{ik} {\hat h}_i(W^{(0)})\big\|_2\Big],\\
&\leq \frac{\sqrt{2}B_2}{m} {\mathbb E}_\sigma \Big[\sum^K_{k=1} \sup_{\|W^{(0)}\|_F\leq B_1} \big\| \sum^m_{i=1} \sigma_{ik} {\hat h}_i(W^{(0)})\big\|_2\Big],
\end{align*} 
the last inequality follows by the property that $\sup(\sum_s a_s)\leq \sum_s \sup(a_s)$.  Now the key point is how to estimate the term $\sup_{\|W^{(0)}\|_F\leq B_1} \big\| \sum^m_{i=1} \sigma_{ik} {\hat h}_i(W^{(0)})\big\|_2$.
We will employ the idea introduced in \cite{shaogao2021} (in the proof of Theorem 1) to remove the ``sup'' term. Let
$h_v(W^{(0)})$ $=$ $\Big(h^1_v({\bf w}^{(0)}_1)$, $h^2_v({\bf w}^{(0)}_2)$, $\cdots$, $h^H_v({\bf w}^{(0)}_H)\Big)$ with  $W^{(0)}$ $=$ $({\bf w}^{(0)}_1$,$\cdots$, $w^{(0)}_H)$ and notice that
${\hat h}(W^{(0)})={\hat A} h(W^{(0)})$, then we have
\begin{align*}
\Big\|\sum^m_{i=1}\sigma_{ik}{\hat h}_i(W^{(0)})\Big\|^2_2
&= \sum^H_{t=1}\Big(\sum^m_{i=1}\sigma_{ik}\sum_{v\in N(i)}{\hat A}_{iv}h^t_v({\bf w}^{(0)}_t)\Big)^2_2\\
&=\sum^H_{t=1}\|{\bf w}^{(0)}_t\|^2_2 \Big(\sum^m_{i=1}\sigma_{ik}\sum_{v\in N(i)}{\hat A}_{iv}h^t_v({\bf w}^{0}_t/\|{\bf w}^{(0)}_t\|_2)\Big)^2.
\end{align*}
By the definition of Frobenius norm $\|W\|^2_F= \sum^H_{t=1} \|{\bf w}_t\|^2_2$, the supremum of the above quantity under the constraint $\|W\|_F\leq R$ must be obtained when $\|{\bf w}_{t_0}\|_2=B_1$ for some $t_0\in[H]$, and $\|{\bf w}_t\|_2=0$ for all $t\neq t_0$. Hence
\begin{align*}
\sup_{\|W^{(0)}\|_F\leq B_1}\Big\|\sum^m_{i=1}\sigma_{ik}{\hat h}_i(W^{(0)})\Big\|_2
&=\sup_{\|{\bf w}\|_2=B_1}\Big(\sum^m_{i=1}\sigma_{ik}\sum_{v\in N(i)}{\hat A}_{iv}h_v({\bf w})\Big).
\end{align*}
Let $n_s(j)$ be the $s-$th neighbor number of node $j$ ($s\in [q]$, $j\in [m]$). Recall $q:=\max |N(j)|, j\in[m]$, $M_s =\max_{i\in [m]} |\hat A_{iv}|$ with $v\in N(i)$, therefore
\begin{align*}
&{\mathbb E}_\sigma \Big[\sum^K_{k=1} \sup_{\|W^{(0)}\|_F\leq B_1} \big\| \sum^m_{i=1} \sigma_{ik} {\hat h}_i(W^{(0)})\big\|_2\Big]\\
&={\mathbb E}_\sigma \Big[\sum^K_{k=1}\sup_{\|{\bf w}\|_2=B_1}\Big(\sum^m_{i=1}\sigma_{ik}\sum_{v\in N(i)}{\hat A}_{iv}h_v({\bf w})\Big)\Big]\\
&\leq  {\mathbb E}_\sigma \Big[\sum^K_{k=1}\sup_{\|{\bf w}\|_2=B_1}\Big(\sum^q_{s=1} M_s \sum^m_{i=1} \sigma_{ik} h_{n_s(i)} ({\bf w})\Big)\Big].
\end{align*}

Applying the conclusion $\sup(\sum_s a_s)\leq \sum_s \sup(a_s)$ and contraction property of Rademacher complexity again, we have
\begin{align*}
&{\mathbb E}_{\sigma}\Big[\sum^K_{k=1} \sup_{\|{\bf w}\|_2=B_1} \Big(\sum^q_{s=1} M_s \sum^m_{i=1}\sigma_{ik}h_{n_s(i)}({\bf w})\Big)\Big]\\
&\leq {\mathbb E}_{\sigma}\Big[\sum^K_{k=1} \sum^q_{s=1} M_s\sup_{\|{\bf w}\|_2=B_1} \sum^m_{i=1}\sigma_{ik}h_{n_s(i)}({\bf w})\Big]\\
& \leq \sum^q_{s=1} M_s {\mathbb E}_{\sigma}\Big[\sum^K_{k=1} \sup_{\|{\bf w}\|_2=B_1}  \sum^m_{i=1}\sigma_{ik} \Big(\sum_{j\in N(n_s(i))}{\hat A}_{n_s(i)j}\langle {\bf x}_j, {\bf w}\rangle\Big)\Big]\\
& = \sum^q_{s=1} M_s {\mathbb E}_{\sigma}\Big[\sum^K_{k=1} \sup_{\|{\bf w}\|_2=B_1}   \Big\langle \sum^m_{i=1}\sigma_{ik}\sum_{j\in N(n_s(i))}{\hat A}_{n_s(i)j}{\bf x}_j, {\bf w}\Big\rangle\Big]\\
&\leq  B_1\sum^q_{s=1} M_s {\mathbb E}_{\sigma}\Big[\sum^K_{k=1}\Big \|\sum^m_{i=1}\sigma_{ik}\sum_{j\in N(n_s(i))}{\hat A}_{n_s(i)j}{\bf x}_j\Big\|_2\Big].\\
\end{align*}
Therefore, we only need to estimate the term
$$ {\mathbb E}_{\sigma}\Big[\sum^K_{k=1}\Big \|\sum^m_{i=1}\sigma_{ik}\sum_{j\in N(n_s(i))}{\hat A}_{n_s(i)j}{\bf x}_j\Big\|_2\Big].$$
Applying Cauchy-Schwartz inequality yields that
\begin{align*}
{\mathbb E}_{\sigma}\Big[\sum^K_{k=1}\Big \|\sum^m_{i=1}\sigma_{ik}\sum_{j\in N(n_s(i))}{\hat A}_{n_s(i)j}{\bf x}_j\Big\|_2\Big]
&\leq \sqrt{{\mathbb E}_{\sigma}\Big(\sum^K_{k=1}\Big \|\sum^m_{i=1}\sigma_{ik}\sum_{j\in N(n_s(i))}{\hat A}_{n_s(i)j}{\bf x}_j\Big\|_2\Big)^2}\\
&\leq  \sqrt{{\mathbb E}_{\sigma}K \sum^K_{k=1} \Big\|\sum^m_{i=1}\sigma_{ik}\sum_{j\in N(n_s(i))}{\hat A}_{n_s(i)j}{\bf x}_j\Big\|^2_2}\\
&\leq K\sqrt{\sum^m_{i=1}\Big\|\sum_{j\in N(n_s(i))}{\hat A}_{n_s(i)j}{\bf x}_j\Big\|_2^2},
\end{align*}
where the last inequality is due to the i.i.d condition of Rademacher sequences.  Plugging the conclusion of Lemma \ref{critial lemma} into the above term leads to
$${\mathbb E}_{\sigma}\Big[\sum^K_{k=1} \Big\|\sum^m_{i=1}\sigma_{ik}\sum_{j\in N(n_s(i))}{\hat A}_{n_s(i)j}{\bf x}_j\Big\|_2\Big] \leq KR \sqrt{qm},$$
and
$${\cal {\hat R}} ({{\cal F}}_{B_1,B_2}) \leq \frac{\sqrt{2q} KB_1B_2 R\sum^q_{s=1}M_s}{\sqrt{m}}.$$
This completes the proof by combining with Lemma \ref{rademacher_ineq}.
\end{proof}

\section*{Acknowledgement}
The work described in this paper was supported partially by the National Natural Science Foundation of China (12271111, 11871277), Special Support Plan for High-Level Talents of Guangdong Province (2019TQ05X571), Guangdong Basic and Applied Basic Research Foundation (2022A1515011726). The authors would like to thank Prof. Hong Chen from Huazhong Agricultural University for useful discussions about the theoretical analysis, which have helped to improve the presentation of the paper.
\bibliographystyle{abbrv}

\end{document}